\documentclass[11pt]{article}
\usepackage{lipsum}

\newcommand\blfootnotea[1]{\begingroup
  \renewcommand\thefootnote{}\footnote{#1}\endgroup
}

\usepackage{amsthm,amsmath,amssymb,amsfonts,amssymb,mathtools}

\usepackage{algorithm2e}

\newif\ifmarkup
\markupfalse

\usepackage{thmtools}
\usepackage{thm-restate}
\usepackage{aliascnt}
\newtheorem{theorem}{Theorem}[section]

 \newtheorem{lemma}[theorem]{Lemma}
 \newtheorem{informal theorem}[theorem]{Theorem (informal statement)}
 
 \newtheorem{proposition}[theorem]{Proposition}
 \newtheorem{corollary}[theorem]{Corollary}
 \newtheorem{claim}[theorem]{Claim}
 \newtheorem{fact}[theorem]{Fact}

 \newtheorem{definition}[theorem]{Definition}

\usepackage[english]{babel}

\usepackage[letterpaper,margin=1in]{geometry}
\usepackage[utf8]{inputenc}
\usepackage{amsmath}
\usepackage{graphicx}
\usepackage{xspace}
\usepackage{caption}

\usepackage{natbib}

\newif\ifcomments
\commentstrue

\ifcomments
\newcommand{\jd}[1]{{\color{purple}{\textbf{JD:} #1}}}
\newcommand{\pq}[1]{{\color{brown}{\textbf{PQ:} #1}}}
\newcommand{\nz}[1]{{\color{violet}{\textbf{Zn:} #1}}}
\else
\newcommand{\jd}[1]{}
\newcommand{\pq}[1]{}
\newcommand{\nz}[1]{}
\fi

\usepackage{enumitem}

\newcommand{\subscript}[2]{$#1 _ #2$}
\newcommand{\mcal}{\mathcal}
\newcommand{\barw}{\Bar{\w}}
\newcommand{\barx}{\Bar{\x}}
\newcommand{\evt}{\mathcal{E}}

\usepackage[colorlinks,citecolor=blue,linkcolor=magenta,bookmarks=true]{hyperref}
\usepackage[nameinlink,capitalize]{cleveref}
\crefname{claim}{claim}{claims}
\crefname{fact}{fact}{facts}
\usepackage{mysty}

\title{Information-Computation Tradeoffs for Learning Margin Halfspaces with Random Classification Noise\blfootnotea{Authors are in alphabetical order.}
}

\author{
Ilias Diakonikolas\thanks{Supported by NSF Medium Award CCF-2107079,
		NSF Award CCF-1652862 (CAREER), and
		a DARPA Learning with Less Labels (LwLL) grant.}\\
		UW Madison\\
		{\tt ilias@cs.wisc.edu}\\
\and Jelena Diakonikolas\thanks{Supported by NSF Award CCF-2007757 and by the U.\ S.\ Office of
Naval Research under award number N00014-22-1-2348.}\\
UW Madison\\
{\tt jelena@cs.wisc.edu}\\
\and Daniel M. Kane\thanks{Supported by NSF Medium Award CCF-2107547 and NSF Award CCF-1553288 (CAREER).}\\
UC San Diego\\
{\tt dakane@ucsd.edu}
\and Puqian Wang\thanks{Supported in part by NSF Award CCF-2007757.}\\
UW Madison\\
{\tt pwang333@wisc.edu}\\
\and 
Nikos Zarifis\thanks{Supported in part by NSF award 2023239,  NSF Medium Award CCF-2107079, and a DARPA Learning with Less Labels (LwLL) grant.}\\
UW Madison\\
{\tt zarifis@wisc.edu}\\
}

\begin{document}
\maketitle

\begin{abstract}
 We study the problem of PAC learning $\gamma$-margin halfspaces 
with Random Classification Noise. 
We establish an information-computation tradeoff
suggesting an inherent gap between the sample complexity of the problem 
and the sample complexity of computationally efficient algorithms. 
Concretely, the sample complexity of the problem is $\widetilde{\Theta}(1/(\gamma^2 \eps))$. We start by giving a simple efficient algorithm 
with sample complexity $\widetilde{O}(1/(\gamma^2 \eps^2))$. Our main result
is a lower bound for Statistical Query (SQ) algorithms and low-degree polynomial tests suggesting that the quadratic dependence on $1/\eps$ 
in the sample complexity is inherent for computationally efficient algorithms.
Specifically, our results imply a lower bound of $\widetilde{\Omega}(1/(\gamma^{1/2} \eps^2))$ on the sample complexity of any efficient SQ learner or low-degree test.
\end{abstract}

\setcounter{page}{0}
\thispagestyle{empty}
\newpage

\section{Introduction}

\new{This work studies the efficient learnability of halfspaces with a margin in the presence of random label noise. Before we present our contributions, we provide the necessary background.} 
A halfspace or Linear Threshold Function (LTF) is any 
Boolean-valued function $h: \R^d \to \{ \pm 1\}$ of the form 
$h(\bx) = \sgn \left( \bw \cdot \bx -\theta \right)$, where $\bw \in \R^d$ is the weight vector and $\theta \in \R$ is the threshold. 
The function $\sign: \R \to \{ \pm 1\}$ is defined as $\sgn(t)=1$ if $t \geq 0$ 
and $\sgn(t)=-1$ otherwise.
The problem of learning halfspaces with a margin --- i.e., under the assumption 
that no example lies too close to the separating hyperplane --- is a textbook problem in machine learning, whose history goes back to the Perceptron algorithm of~\citet{Rosenblatt:58}. 
Here we study the problem of PAC learning margin halfspaces in  the presence of Random Classification Noise (RCN)~\citep{AL88}.

Before we describe the noisy setting (the focus of this work), 
we recall the basics in the realizable PAC model~\citep{val84}
(i.e., when the labels are consistent with the target concept). 
We will henceforth assume that the threshold is $\theta = 0$, 
which is well-known to be no loss of generality.
The setup is as follows:
there is an unknown distribution $(\bx, y) \sim \D$ on 
$\mathbb{S}^{d-1} \times \{ \pm 1\}$, 
where $\mathbb{S}^{d-1}$ is the unit sphere on $\R^d$,  
such that $y = \sgn(\wstar \cdot \x)$ for some $\wstar \in \R^d$
with $\|\wstar\|_2 = 1$. The margin assumption means 
that the marginal distribution of $\D$ on the examples $\x$, denoted by $\D_{\x}$, 
puts no probability mass on points with distance less than $\gamma \in (0, 1)$ 
from the separating hyperplane $\wstar \cdot \x = 0$; that is, 
we have that $\pr_{\x \sim \D_{\x}}\left[ |\wstar \cdot \x| < \gamma \right] = 0$. 
The parameter $\gamma$ is called the margin of the target halfspace. 

In this context, the learning algorithm 
is given as input a desired accuracy $\eps>0$
and a training set $S = \{(\bx^{(i)}, y^{(i)}) \}_{i=1}^m$ of i.i.d.\ samples from $\D$. 
The goal is to output a hypothesis $h: \R^d \to \{\pm 1\}$ 
whose misclassification error $\err_{0-1}^{\D}(h) \eqdef \pr_{(\bx, y) \sim \D}[h(\bx) \neq y]$ 
is at most $\eps$ with high probability\footnote{Throughout this introduction, we will take the failure probability to be a small constant, say $1/10$.}. The aforementioned setting is well-understood. First, it is known that
the sample complexity of the learning problem, ignoring 
computational considerations, is $\Theta(1/(\gamma^2\eps))$; see, e.g.,~\citet{SB14}\footnote{We are implicitly assuming 
that $d = \Omega (1/\gamma^2$); otherwise, a sample complexity 
bound of $\widetilde{O}(d/\eps)$ follows from standard VC-dimension arguments.}. Moreover, the Perceptron algorithm is a computationally efficient PAC learner achieving this sample complexity. (This follows, e.g., by combining the mistake bound of $O(1/\gamma^2)$ of the online Perceptron with the online-to-PAC conversion in~\citet{Littlestone:89b}.) 
That is, {\em in the realizable setting, there exists a computationally efficient learner for margin halfspaces 
achieving the 
optimal sample complexity} (within constant factors). 

The high-level question that serves as the motivation for this work is the following: 
\begin{center}
{\em Can we develop ``similarly efficient'' algorithms in the presence of label noise,
and specifically in the (most) basic model of Random Classification Noise?}
\end{center}
By the term ``similarly efficient'' above, we mean that we would like
a polynomial-time algorithm with near-optimal sample complexity (up to logarithmic factors). 

\new{This problem appears innocuous and} our initial efforts focused towards obtaining such an algorithm. 
After several failed attempts, we established  
an information-computation tradeoff strongly suggesting 
that such an algorithm does not exist. 
We next describe our setting in more detail.

\vspace{-0.3cm}

\paragraph{Learning Margin Halfspaces with RCN}
The RCN model~\citep{AL88} is  
the most basic model of random label noise.
In this model, the label of each example is independently 
flipped with probability exactly $\eta$,
where $0< \eta<1/2$ is a noise parameter. 
Since its introduction, RCN has been studied extensively 
in learning theory from both an information-theoretic 
and an algorithmic standpoint. 
One of the early fundamental results in this field was given by~\citet{Kearns:98}, 
who showed that any Statistical Query (SQ) algorithm can be transformed 
into a PAC learning algorithm that is tolerant to RCN. This 
transformation preserves statistical and computational efficiency  
{\em within polynomial factors}.

We return to our problem of PAC learning margin halfspaces with RCN. 
The setup is very similar to the one above. 
The only difference is that the labels are now perturbed 
by RCN with noise rate $\eta$ (see \Cref{assmpt:margin+RCN}). 
As a result, the 
optimal misclassification error is equal to $\eta$, and the goal is to find a hypothesis that with 
high probability satisfies $\err_{0-1}^{\D}(h) \leq 
\eta+\eps$. A closely related objective would be to approximate 
the target halfspace, i.e., the function $\sgn(\wstar \cdot \x)$, 
within any desired accuracy $\eps'>0$. It is well-known (and easy to 
derive) that the two goals are essentially equivalent, up to rescaling the parameter $\eps$ by a 
factor of $(1-2\eta)$. In this paper, we phrase our results for the misclassification error with respect to the observed labels.

In this context, the sample complexity of PAC learning margin halfspaces with RCN 
is equal to $\widetilde{\Theta}(1/((1-2\eta)\gamma^2 \eps))$. This bound can be derived, e.g., 
from the work of~\cite{Massart2006}. (That is, the sample complexity of the RCN 
learning problem is essentially the same as in the realizable case ---  assuming $\eta$ is bounded from $1/2$ --- 
within logarithmic factors.)
On the algorithmic side, a number of works, 
starting with~\citet{Bylander94}, developed polynomial sample and time algorithms 
for this learning task. Specifically, \citet{Bylander94} developed a careful adaptation 
of the Perceptron algorithm for this purpose.  Subsequently,~\citet{BFK+:97} pointed out 
that an SQ version of the Perceptron algorithm coupled with Kearns' reduction 
immediately implies the existence of an efficient RCN learner 
(see also~\citet{Cohen:97} for a closely related work). 
More recently, in a related context,~\cite{DGT19} pointed out that a simple convex surrogate loss can be used for this purpose \new{(see also~\citet{DKTZ20} for a related setting)}. 

The preceding paragraph might suggest that the RCN version of the problem is 
fully resolved. The catch is that {\em all known algorithms for the problem 
require sample complexities that are polynomially worse 
than the information-theoretic minimum}.
Specifically, for all known polynomial-time algorithms, 
the dependence of the sample complexity on 
the inverse of the accuracy parameter $\eps$ is {\em at least quadratic}
--- while the information-theoretic minimum scales near-linearly with $1/\eps$. 
It is thus natural to ask 
whether a computationally efficient algorithm 
with (near-)optimal sample complexity exists. 
This leads us to the following question:
\begin{center}
{\em Is the existing gap between the sample complexity of known efficient algorithms\\
and the information-theoretic sample complexity inherent?}
\end{center}
In this paper, we resolve the above question in the affirmative  
for a broad class of algorithms --- specifically, for all 
Statistical Query algorithms and low-degree polynomial tests.

\subsection{Our Results} \label{sec:results}

The following definition summarizes our setting.

\begin{definition}[PAC Learning Margin Halfspaces with RCN]\label{assmpt:margin+RCN}
Let $\D$ be a distribution over $\mathbb{S}^{d-1}\times\{\pm 1\}$, where $\mathbb{S}^{d-1}$ is the unit sphere in $\R^d,$ and let $\wstar\in\mathbb{S}^{d-1}$. Let $\gamma\in(0, 1)$ and $\eta \in (0, \frac{1}{2})$. 
For each sample $(\x,y)\sim \D$, the following assumptions both hold: \begin{enumerate}[label=(\subscript{A}{{\arabic*}})]
    \item \label{assmpt:margin} The unit vector $\wstar$ satisfies the $\gamma$-margin condition, i.e.,  $\pr_{(\x,y) \sim \D} \left[ |\wstar \cdot \x| < \gamma \right] = 0$.
\item \label{assmpt:RCN} For each point $\x \in \mathbb{S}^{d-1}$, the corresponding label $y$ 
    satisfies: with probability $1-\eta$,  $y=\sgn(\wstar \cdot\x)$; 
    otherwise, $y=-\sgn(\wstar \cdot\x)$.

\end{enumerate}
Given i.i.d.\ samples from $\D$, the goal of the learner is to output a hypothesis $h$ 
that with high probability satisfies 
$\err_{0-1}^{\D}(h) \eqdef \pr_{(\bx, y) \sim \D}[h(\bx) \neq y] \leq \eta+\eps$.
\end{definition}

\noindent While our definition applies to homogeneous halfspaces, 
the case of general halfspaces, \new{i.e., functions of the form $y=\sign(\wstar \cdot \x -t)$,} 
can easily be reduced to the homogeneous case 
by increasing the dimension by one, i.e., 
writing $y=\sgn((\wstar,-t)\cdot (\x,1))$. 
By rescaling these vectors and letting 
$\vec w':=(\wstar,t)/\sqrt{1+t^2}$, and $\x':=(\x,1)/\sqrt{2},$ 
we have a homogeneous halfspace $y=\sgn(\vec w'\cdot \x')$ 
in $d+1$ dimensions with margin $\gamma/\sqrt{2(1+t^2)} \geq \gamma/2$ 
equivalent to our original problem. 
By the homogeneity assumption, it follows that the assumption that 
the examples lie on the unit sphere is also generic (up to scaling). 
We also recall that it can be assumed 
without loss of generality
that the noise rate $\eta$ is known to the algorithm~\citep{AL88}.

To the best of our knowledge, prior to our work, 
the best known sample-complexity upper 
bound of an efficient algorithm for our problem was
$\widetilde{O}(1/(\gamma^4 \eps^2))$~\citep{DGT19}.
Our first result is a computationally efficient learner
with sample complexity  $\widetilde{O}(1/(\gamma^2 \eps^2))$. 

\begin{theorem}[Algorithmic Result] \label{thm:alg-inf}
There exists an algorithm that draws $N = \widetilde{O}(1/(\gamma^2 \eps^2))$ samples, runs in time 
$\poly(N, d)$ and learns $\gamma$-margin halfspaces up to misclassification error $\eta+\eps$ with probability at least $9/10$.
\end{theorem}

See \Cref{thm:main-algo} for a more detailed formal statement. 
While the above sample bound does not improve on
the $\eps$-dependence over prior algorithmic results, 
it does improve the dependence on the margin parameter $\gamma$ 
quadratically  --- nearly matching the information-theoretic lower bound 
(within logarithmic factors). 
\new{An independent and contemporaneous work by \citet{KIT23} obtained a similar $O_{\eta}(1/(\gamma^2 \epsilon^2))$ sample complexity result for learning $\gamma$-margin halfspaces with RCN in polynomial time, using a different algorithm and techniques.}

Our second and main result is an information-computation
tradeoff suggesting that the quadratic dependence in $1/\eps$ 
is inherent for polynomial-time algorithms. Formally, we establish 
such a \new{tradeoff} in the Statistical Query model 
and (via a known reduction) for low-degree polynomial tests.

\paragraph{Statistical Query (SQ) Model} 
Before we state our main result, we recall the basics of the SQ model~\citep{Kearns:98}. 
Instead of drawing samples from the input distribution, 
SQ algorithms  are given query access to the distribution 
via the following oracle:

\begin{definition}[STAT Oracle] \label{def:stat}
Let $\D$ be a distribution on $\R^d$. 
A statistical query is a bounded function $f: \R^d \to [-1,1]$. 
For tolerance $\tau>0$ of the statistical query, the $\mathrm{STAT}(\tau)$ oracle responds to the query $f$ 
with a value $v$ such that $|v - \E_{\bx \sim D}[f(\bx)]| \leq \tau$. 
\end{definition}

\noindent We note that other oracles have been considered in the literature, in particular $\mathrm{VSTAT}$ (\Cref{def:vstat}); our lower bound also holds with respect to these oracles.

An SQ lower bound for a learning problem $\Pi$ is typically of the following form: 
any SQ algorithm for $\Pi$ must either make at least $q$ queries
or it makes at least one query with small tolerance $\tau$. 
When simulating a statistical query in the standard PAC model 
(by averaging i.i.d.\ samples to approximate expectations), 
the number of samples needed for a $\tau$-accurate query 
can be as high as $\Omega(1/\tau^2)$. Thus, we can intuitively interpret 
an SQ lower bound as a tradeoff between runtime of $\Omega(q)$ 
and a sample complexity of $\Omega(1/\tau^2)$.

We are now ready to state our SQ lower bound:

\begin{theorem}[SQ Lower Bound] \label{thm:sq-inf}
For any constant $c>0$, any SQ algorithm 
that learns $\gamma$-margin halfspaces on the unit sphere 
in the presence of RCN with $\eta = 1/3$ to error $\eta+\eps$ 
requires at least $2^{(1/\gamma)^{\Omega(c)}}$ queries
or makes at least one query with tolerance $O(\eps \gamma^{1/4-c})$.
\end{theorem}

The reader is referred to \Cref{thm:sq-theorem} 
for a more detailed formal statement. 
The intuitive interpretation of our result is that any 
(sample simulation of an) SQ algorithm for the class of 
$\gamma$-margin halfspaces \new{with RCN} 
either draws at least $\Omega(1/(\gamma^{1/2-c}\eps^2))$ samples or requires at least $2^{(1/\gamma)^{\Omega(c)}}$ time. 
That is, for sufficiently small $\eps$ (namely, $\eps \leq \gamma^{3/2 + c}$), 
the computational sample complexity of the problem (in the SQ model) 
is polynomially higher than its information-theoretic sample complexity. See \Cref{thm:hypothesis-testing-hardness}
for the implications to low-degree polynomial tests.

\new{Finally, we note that SQ lower bounds have been previously obtained for the (more challenging) problem of learning halfspaces with bounded (Massart) noise in a variety of regimes~\citep{DK20-SQ-Massart, DiakonikolasKKT22, NasserT22}. Importantly, all these previous results make essential use of the bounded noise model and do not apply in the RCN setting.}

\subsection{Our Techniques} \label{ssec:techniques}

\paragraph{Upper Bound}
Our algorithmic approach is quite simple: 
we use projected subgradient descent 
applied to the leaky ReLU loss 
with parameter $\eta$, 
as was done in previous work~\citep{DGT19}. 
However, our analysis never explicitly makes a 
connection to minimizing the leaky ReLU loss; for 
our arguments, this loss is irrelevant. Instead, 
we make a novel connection between the 
(sub)gradient field of the leaky ReLU loss and the 
disagreement between how the vector $\w$ at which 
the subgradient is evaluated and an optimal vector 
$\wstar$ would classify points. Through this 
connection, we leverage the regret analysis of 
projected subgradient descent to obtain a novel 
regret bound on the disagreement probability 
$\pr[\sign(\w_t \cdot \x) \neq \sign(\wstar \cdot \x)],$ 
where $\w_t$ are the iterates of the algorithm. 
The obtained bound decomposes into three terms: 
(i) the standard regret term, which is bounded by choosing the algorithm iteration count to be sufficiently high, 
(ii) an error term bounded by the subgradient norm of empirical leaky ReLU at $\wstar,$ 
and (iii) an error term that corresponds to the uniform convergence error of the disagreement function $\1\{\sign(\w\cdot\x) \neq \sign(\wstar\cdot \x)\}.$ 
The latter two terms are then bounded choosing the sample size to be sufficiently high, yet bounded by $\widetilde{O}(1/(\eps^2 \gamma^2))$.

\paragraph{SQ Lower Bound}
To prove our SQ lower bound, we bound from below the 
SQ dimension of the problem; using standard results~\citep{FeldmanGRVX17} this implies the desired lower bound guarantees. 
Bounding the SQ dimension from below amounts to establishing the existence of a large set of distributions whose pairwise correlations are small. Inspired by the technique of~\citet{DKS17-sq}, we achieve this by selecting our distributions 
to be random rotations of a single distribution, 
each behaving in a standard way in all but one critical 
direction. To ensure the necessary margin property, 
we make the $\x$-marginals of our distribution 
uniform over the hypercube in $d \ll 1/\gamma$ dimensions --- as 
opposed to Gaussian-like (as in~\citet{DKS17-sq}).

The distributions we consider are quite simple. 
We define $f_{\bv}(\x) = \sgn(\bv \cdot \x  - t)$, 
where $\bv$ is a randomly chosen Boolean-valued vector 
and the threshold $t$ is chosen so that the probability 
that $\bv \cdot \x > t$ is of the order of $\eps$. 
We then let $\x$ be the uniform distribution over the hypercube 
and let $y = f_{\bv}(\x)$ with probability $2/3$  
and $-f_{\bv}(\x)$ otherwise. 
By picking many different vectors $\bv$, 
we get many different LTFs. 
We claim that there exist many of these LTFs 
whose pairwise correlations 
(with respect to the distribution where $\x$ 
is independent over the hypercube and $y$ is independent of $\x$) 
are small, as long as the corresponding 
defining vectors $\bv$ and $\bv'$ have small inner product 
(see \Cref{lem:a-b-correlation} and \Cref{lemma:correlation-bound}). 
Intuitively, this should hold because 
(i) both distributions are already $\eps$-close 
to the base distribution, and 
(ii) when $\bu$ and $\bv$ are nearly orthogonal, 
$f_{\bv}(\x)$ and $f_{\bu}(\x)$ are nearly 
independent of each other.

To analyze this inner product, 
we use a Fourier analytic approach. 
First, we note that the sizes of the individual Fourier coefficients of $f_{\bu}$ and $f_{\bv}$ can be 
computed using \new{Kravchuk} 
polynomials (see \Cref{claim:fourier-coef}). 
This allows us to show that they do not have too much 
Fourier mass in low degrees. Second, 
we note that when taking the inner product 
of the degree-$k$ parts of the Fourier transforms of 
$f_{\bu}$ and $f_{\bv}$, 
we will have large amounts of cancellation of terms, 
particularly when $|\bu \cdot \bv|$ is small 
or when $k$ is large (see \Cref{clm:lower-bound-large-k} and \Cref{clm:lower-bound-small-k}). 
The size of the remaining term after the cancellation can be written 
in terms of another \new{Kravchuk} polynomial, which we can bound. 
A careful analysis of all of the relevant terms  
gives us the necessary correlation bounds which imply our main result.

\vspace{-0.2cm}

\subsection{Notation}\label{sec:prelims}
For $n \in \Z_+$, we use $[n]$ to denote the set $\{1, \ldots, n\}$.
We use small
boldface characters for vectors and capital bold characters for matrices.  For $\bx \in \R^d$, $\|\bx\|_2 \eqdef (\littlesum_{i=1}^d \bx_i^2)^{1/2}$ denotes the $\ell_2$-norm of $\bx$.
We use $\bx \cdot \by $ for the inner product of $\bx, \by \in \R^d$ and $\mathcal{B}_d = \{\x \in \R^d: \|\x\|_2\leq 1\}$ to denote the unit centered Euclidean ball in $\R^d$; when the dimension is clear from the context, we omit it from the subscript. We use $\1\{A\}$ to denote the indicator function of $A$; equal to one if $A$ is a true statement, and equal to zero otherwise. 

We use $\E_{x\sim \D}[x]$ for the expectation of the random variable $x$ according to the distribution $\D$ and
$\pr[\mathcal{E}]$ for the probability of event $\mathcal{E}$. For simplicity of notation, we may
omit the distribution when it is clear from the context. For $(\x,y)$ distributed according to $\D$, we denote by $\D_\x$  the marginal distribution of $\x$ and by $\D_y$  the marginal distribution of $y$. We denote by $\mathcal U_d$ the uniform distribution over $\{\pm 1\}^d$.

\section{Computationally Efficient Learning Algorithm}\label{sec:algo}

In this section, we give the algorithm establishing \Cref{thm:alg-inf}.
We start by providing some intuition for our algorithm and its analysis. We then formally state the algorithm and bound its sample complexity and runtime. Due to space constraints, some of the technical details and proofs are deferred to \Cref{appx:omitted-pfs-algo}. 

\paragraph{Leaky ReLU, its subgradient, and intuition.} The Leaky ReLU loss function with parameter $\eta \in (0, 1/2)$ is defined by
\begin{equation}\label{eq:leaky-ReLU-def}
    \LR_\eta(z) :=(1-\eta)z \1\{z\geq 0\} + \eta z \1\{z<0\} \;.
\end{equation}
While the Leaky ReLU has been used as a convex surrogate for margin halfspace classification problems, this is not the core of our approach: we never argue about minimizing the expected leaky ReLU loss nor that its minimizer is a good classifier. Instead, we rely on the following vector-valued function $\g_{\eta}: \R^d \to \R^d$ 
\begin{equation}\label{eq:Leaky-ReLU-subgrad}
    \g_{\eta}(\w; \x, y) = \frac{1}{2}\big[(1-2\eta)\sign(\w\cdot \x) - y\big]\x \;.
\end{equation}
When $\x$ and $y$ are clear from the context, we omit them and instead simply write $\g_{\eta}(\w).$ 
The connection between $\g_{\eta}$ and $\LR$ is that $\g_{\eta}(\w)$ is a subgradient of $\LR(\w);$ see, e.g., \citet[Lemma 2.1]{DGT19}. However, this connection is not important for our analysis and we do not make any explicit use of the leaky ReLU function itself. 
What we do rely on is the following key observation.

\begin{proposition}\label{prop:grad-to-misclassification}
    For any $\w, \wb \in \R^d,$ 
    $$
        (\g_{\eta}(\w) - \g_{\eta}(\wb))\cdot( \w - \wb) = (1 - 2\eta) \1\{\sign(\w\cdot \x) \neq \sign(\wb \cdot \x)\}(|\w\cdot \x| + |\wb \cdot \x|). 
    $$
\end{proposition}
\begin{proof}
    By a direct calculation,
\begin{align*}
        ({\vec g_\eta}(\w) - {\vec g_\eta}(\wb))\cdot( \w - \wb) = \frac{1 - 2\eta}{2} \big(&(\sign(\w\cdot \x) - \sign(\wb\cdot \x))(\w\cdot \x - \wb \cdot \x)\big)\\
        = \frac{1 - 2\eta}{2} \big(&|\w\cdot \x| + |\wb\cdot \x|\\
        &- (\w\cdot\x ) \, \sign(\wb\cdot \x) - (\wb\cdot\x) \, \sign(\w\cdot \x)\big). 
    \end{align*}
In the last expression, the term in the parentheses is zero when the signs of $\w\cdot \x$ and $\wb \cdot \x$ agree; otherwise it is equal to $2(|\w\cdot \x| + |\wb\cdot \x|)$, leading to the claimed identity.  
\end{proof}

In particular, recalling that $\wstar$ is the weight vector of the target halfspace (see \Cref{assmpt:margin+RCN}), \Cref{prop:grad-to-misclassification} implies that 
\begin{equation}\label{eq:grad-to-misclassification}
    (\g_{\eta}(\w) - \g_{\eta}(\wstar))\cdot( \w - \wstar) = (1 - 2\eta) \1\{\sign(\w\cdot \x) \neq \sign(\wstar \cdot \x)\}(|\w\cdot \x| + |\wstar \cdot \x|). 
\end{equation}
In other words, the inner product $(\g_{\eta}(\w) - \g_{\eta}(\wstar))\cdot( \w - \wstar)$ is proportional to the Boolean function $\1\{\sign(\w\cdot \x) \neq \sign(\wstar \cdot \x)\}$ that indicates disagreement between $\w$ and $\wstar$ in how they classify points $\x.$ In particular, if we argue that $\Exx[\1\{\sign(\w\cdot \x) \neq \sign(\wstar \cdot \x)\}] = \pr[\sign(\w\cdot \x) \neq \sign(\wstar \cdot \x)] \leq \Bar{\epsilon}$ for some $\w,$ then we can immediately conclude that the misclassification error of $\w$ is $\eta + (1-2\eta)\Bar{\eps}$, due to \Cref{assmpt:margin+RCN}, \Cref{assmpt:RCN}. Thus, for $\Bar{\eps} = \frac{\eps}{1 - 2\eta},$ the misclassification error is $\eta + \eps$. This is the approach that we take.

To carry out the analysis, given $\x^{(1)}, \x^{(2)}, \dots \x^{(N)}$ drawn i.i.d.\ from $\D_{\x},$ we use
\begin{equation}\label{eq:emp-prob-disagreement-def}
    \widehat{\pr}_N(\w) := (1/N) \sum_{i=1}^N \1\{\sign(\w\cdot \x^{(i)}) \neq \sign(\wstar \cdot \x^{(i)})\}
\end{equation}
to denote the empirical probability of disagreement between $\w$ and $\wstar.$ 

\paragraph{Projected subgradient descent and disagreement regret.} Our \Cref{algo:main} is the simple projected subgradient descent, applied to the subgradient of the empirical leaky ReLU function, which given a sample $\{(\x^{(i)}, y^{(i)})\}_{i=1}^N$ drawn i.i.d.\ from $\D$ is defined by
\begin{equation}\label{eq:emp-grad}
    \hgn(\w) := (1/N) \sum_{i=1}^N \g_{\eta}(\w; \x^{(i)}, y^{(i)}).
\end{equation}
\RestyleAlgo{ruled}
\begin{algorithm}[ht]
\DontPrintSemicolon
\KwIn{$\epsilon > 0$, $\gamma \in (0, 1),$ $\eta \in (0, 1/2),$ i.i.d.\ sample $\{\x^{(i)}, y^{(i)}\}_{i=1}^N$  from $\D$, any $\vec w_0 \in \mathcal{B}$} Let $T = \big\lceil \frac{16(1-\eta)^2}{\gamma^2 \epsilon^2} - 1 \big\rceil,$ $\mu = \frac{2}{(1-\eta)\sqrt{T+1}}$\;
\For{$t = 0:T-1$}{
$\hat{\vec g}_N(\w_t) = \frac{1}{2N}\sum_{i=1}^N \big((1-2\eta)\sign(\w_t \cdot \x^{(i)}) - y^{(i)}\big)\x^{(i)}$\;

$\w_{t+1} = \proj_{\mathcal{B}}\big(\w_t - \mu \hat{\vec g}_N(\w_t)\big)$, where $\mathcal{B} = \{\x \in \R^d: \|\x\|_2 \leq 1\}$ \;
}
\KwRet{$\{\w_0, \dots, \w_T\}$} 

\caption{PAC Learner for Margin Halfspaces with RCN}\label{algo:main}
\end{algorithm}

To utilize standard regret bounds for (projected) subgradient descent, we first observe that $\g_{\eta}$ is bounded for all $\w \in \R^d.$ As a consequence, $\hgn$ admits the same upper bound.\begin{claim}\label{claim:subgrad-bnd}
    Given any $\w \in \R^d$ and any $(\x, y) \in \mathbb{S}^{d-1}\times \{-1, 1\}$, $\|\g_{\eta}(\w; \x, y)\|_2 \leq 1 - \eta.$ As a consequence, given any set of points $(\x^{(i)}, y^{(i)}) \in \mathbb{S}^{d-1}\times \{-1, 1\},$ $i \in [N],$ $\|\hgn(\w)\|_2 \leq 1- \eta.$
\end{claim}
\begin{proof}
By the definition of $\g_\eta$ from \Cref{eq:Leaky-ReLU-subgrad}, we have
\begin{align*}
    \|\g_{\eta}(\w; \x, y)\|_2 &= (1/2) \big|(1-2\eta)\sign(\w\cdot \x) - y\big|\|\x\|_2\leq (1/2) \big|(1-2\eta) + 1\big|= 1 - \eta,
\end{align*}
where we have used that $\sign(\cdot) \in \{-1, 1\},$ $y \in \{-1, 1\}$ and $\|\x\|_2 = 1.$  The bound on $\|\hgn(\w)\|_2$ follows immediately from this bound, by its definition and Jensen's inequality.
\end{proof}

The following lemma provides what can be interpreted as a regret bound for the disagreement probability {$\pr[\sgn(\w\cdot\x)\neq\sgn(\w^*\cdot\x)]$}. We refer to is as the ``disagreement regret'' and bound it using the regret analysis of projected subgradient descent, combined with \Cref{eq:grad-to-misclassification} and \Cref{assmpt:margin+RCN}, \Cref{assmpt:margin}. Its proof is provided in \Cref{appx:omitted-pfs-algo}. 

\begin{restatable}{lemma}{mainregretbnd}\label{lemma:main-regret-bnd}
    Consider \Cref{algo:main}. There exists $t \in \{0, \dots T\}$ such that 
    \begin{equation*}\pr[\sign(\w_t\cdot \x) \neq \sign(\wstar \cdot \x)] \leq E_1 + E_2 + E_3, 
    \end{equation*}
where $E_1 = \frac{2(1-\eta)}{(1-2\eta)\gamma\sqrt{T+1}}$, $E_2 =  \frac{2}{(1-2\eta)\gamma}\|\hgn(\wstar)\|_2$, and $E_3 = \frac{1}{T+1}\sum_{t=0}^T\big[\pr[\sign(\w_t\cdot \x) \neq \sign(\wstar \cdot \x)] - \widehat{\pr}_N(\w_t)\big]$.
\end{restatable}

In \Cref{lemma:main-regret-bnd}, error $E_1$ is simply the empirical regret, which can be bounded by choosing the number of iterations $T$ in \Cref{algo:main} to be sufficiently high. Errors $E_2$ and $E_3$ determine the sample complexity of our algorithm, and are dealt with in what follows.

\paragraph{Bounding the required number of samples.} We now show that the errors $E_2$ and $E_3$ can be controlled by choosing a sufficiently large sample size $N$. We then combine everything we have shown so far to state our main result on upper bounds in \Cref{thm:main-algo}. 

\begin{restatable}{lemma}{lemmaSampleComplexity}\label{lemma:sample-complexity}
    Let $E_2$ and $E_3$ be defined as in \Cref{lemma:main-regret-bnd}. For any $\Bar{\epsilon} >0, \delta > 0,$ if $N=\Omega\big(\frac{d}{\Bar{\eps}}+\frac{\eta}{(1-2\eta)^2\bar{\eps}^2\gamma^2})\log(\frac{1}{\delta})\big)$, then with probability at least $1 - \delta$ we have $E_2 + E_3 \leq \frac{\bar{\epsilon}}{2}.$
\end{restatable}

We are now ready to state and prove our main upper bound result.

\begin{theorem}\label{thm:main-algo}
    Let $\D$ be a distribution on pairs $(\x, y) \in \mathbb{S}^{d-1}\times\{\pm 1\}$ as in \Cref{assmpt:margin+RCN}. Then, there is an algorithm (\Cref{algo:main}) that for any given $\epsilon, \delta \in (0, 1)$ uses $N=O\big(\frac{d(1-2\eta)}{\eps}+\frac{\eta}{\eps^2\gamma^2})\log(\frac{1}{\delta})\big)$ samples, runs in time $O(\frac{Nd}{\epsilon^2\gamma^2})$ and learns $\gamma$-margin halfspaces corrupted with $\eta$-RCN up to error $\eta +\eps$, with probability at least $1-\delta$.
\end{theorem}
\begin{proof}
    Applying \Cref{lemma:main-regret-bnd} and \Cref{lemma:sample-complexity}, we have that for $N=O\big(\frac{d(1-2\eta)}{\eps}+\frac{\eta}{\eps^2\gamma^2})\log(\frac{1}{\delta})\big)$, with probability at least $1 - \delta,$ there exists  $t \in \{0, \dots, T\}$ in \Cref{algo:main} such that 
\begin{equation}\notag
        \pr[\1\{\sign(\w_t\cdot \x) \neq \sign(\wstar \cdot \x)\}] \leq E_1 + \frac{\epsilon}{2(1-2\eta)}, 
    \end{equation}
where $E_1 = \frac{2(1-\eta)}{(1-2\eta)\gamma\sqrt{T+1}}$. Hence, to ensure that {$\pr[\sgn(\w_t\cdot\x) \neq \sgn(\w^*\cdot\x)]\leq \frac{\eps}{1-2\eta}$,} it suffices to choose $T \geq \frac{16(1-\eta)^2}{\gamma^2 \epsilon^2} - 1,$ which is what \Cref{algo:main} does. The bound on the runtime is then simply $O(TNd),$ as the complexity of each iteration is dominated by the computation of $\hgn,$ which takes $O(Nd)$ time. 
By \Cref{assmpt:margin+RCN}, \Cref{assmpt:RCN}, such a $\w_t$ misclassifies points $\x$ drawn from $\D$ with probability $\eta + \epsilon$. 

    What we have shown so far is that at least one of the vectors $\w_0, \dots, \w_T$ output by \Cref{algo:main} attains the target misclassification error $\eta + \epsilon,$ but we have not specified which one. The appropriate vector can 
be determined by drawing a fresh sample 
$\{(\Tilde{\x}^{(i)}, \Tilde{y}^{(i)})\}_{i=1}^{N'}$ of size $N' = O\big(\log(\frac{T}{\epsilon})\log(\frac{1}{\delta})\big) = O\big(\log(\frac{1}{\epsilon \gamma})\log(\frac{1}{\delta})\big)$ and selecting the vector $\w_t \in \{\w_0, \dots, \w_T\}$ with minimum empirical misclassification error $\frac{1}{N'}\sum_{i=1}^{N'} \1\{\sign(\w_t\cdot \tilde{\x}^{(i)}) \neq \Tilde{y}^{(i)}\}$. Clearly, this additional step does not negatively impact the sample complexity or the runtime stated in \Cref{thm:main-algo}. The standard analysis for this part is provided in \Cref{appx:omitted-pfs-algo}.
\end{proof}

\paragraph{Removing the Dependence on $d$ from the Sample Complexity} In \Cref{thm:main-algo}, the sample complexity $N$ depends on $d$ via the term $\frac{d}{\epsilon}\log(\frac{1}{\delta}),$ which comes from the VC dimension of $O(d)$ that appears when bounding the error term $E_3.$ The dependence on $d$ can be avoided and replaced by $1/\gamma^2,$ using standard dimension-reduction; see \Cref{appx:omitted-pfs-algo}.

\paragraph{Low-Noise Regime} When the noise parameter $\eta$ is equal to zero (i.e., in the realizable setting), 
the sample complexity of the problem is $\Theta(\frac{1}{\gamma^2 \eps})$ and is achievable via the classical Perceptron algorithm. Based on the result of \Cref{thm:main-algo} and with the dimension reduction discussed in the previous paragraph, we recover this optimal sample complexity with our algorithm not only for $\eta = 0,$ but also whenever $\eta = O(\eps).$

%

\section{SQ Lower Bound For Learning Margin Halfspaces with RCN}\label{sec:lowerbound}

In this section, we establish our SQ lower bound result (\Cref{thm:sq-inf})
and its associated implication for low-degree polynomial tests.
 In addition to the $\mathrm{STAT}$ oracle defined in the introduction, we also consider the $\mathrm{VSTAT}$ oracle, defined below. 
\begin{definition}[VSTAT Oracle] \label{def:vstat}
Let $D$ be a distribution on $\R^d$. 
A statistical query is a bounded function $f: \R^d \to [-1,1]$. 
For $t>0$, the $\mathrm{VSTAT}(t)$ oracle responds to the query $f$
with a value $v$ such that $|v - \E_{\bx \sim D}[f(\bx)]| \leq \tau$, where 
$\tau=\max\left(1/t,\sqrt{\Var_{\bx \sim D}[f(\x)]/t}\right)$.
\end{definition}

Our main SQ lower bound result is stated in the following theorem.

\begin{theorem}[Main SQ Lower Bound]\label{thm:sq-theorem} 
Fix $c\in(0,1/2)$. 
Any SQ algorithm that learns \new{the class of} 
$\gamma$-margin halfspaces on $\mathbb{S}^{d-1}$ 
in the presence of RCN with $\eta = 1/3$ 
within misclassification error $\eta+\eps$ either 
requires queries of accuracy better than 
$O(\eps\gamma^{1/4-c/2})$, 
i.e., queries to $\mathrm{STAT}(O(\eps\gamma^{1/4-c/2}))$ 
or $\mathrm{VSTAT}(O(\gamma^{c-1/2}/\eps^2))$, 
or needs to make at least $2^{\Omega(\gamma^{-c})}$ 
statistical queries.
\end{theorem}
\subsection{Proof of \Cref{thm:sq-theorem}}

To prove the theorem, we construct a family of non-homogeneous margin halfspaces 
with RCN such that any SQ learner requires the desired complexity. 
This result can be translated to an SQ lower bound 
for homogeneous halfspaces 
with almost as good margin (see paragraph after \Cref{assmpt:margin+RCN}). 
The family of halfspaces that we construct is supported on 
$\{\pm 1\}^d$ and has margin $\gamma = \Omega(1/d)$.
Note that we can straightforwardly extend this construction  
to a higher dimensional space (by setting the values of the new coordinates 
of points $\x \sim  D_{\x}$ to zero). 
Hence, our construction directly implies a similar 
SQ lower bound for $\gamma$-margin halfspaces 
on the unit sphere $\mathbb{S}^{d-1}$ 
for any $d \gg 1/\gamma$.

Fix $\eps\in(0,1/2)$. For $\vec v\in\{\pm 1\}^d$, let $f_{\vec v}(\x)=\1\{\vec v\cdot \x -t\geq 0\}$ and choose $t\in \R$ so that $\pr[f_{\vec v}(\x)>0]=2\eps$. We define the distribution $D_{\vec v}$ over $\{\pm 1\}^d\times \{0,1\}$ as follows. We choose the marginal distribution of $\x$, denoted by $(D_{\vec v})_\x$, to be the uniform distribution over the set $\{\pm 1\}^d$. For each $\x$, we couple the random variable $y$ by setting $\pr[y=f_{\vec v}(\x)|\x]=1-\eta$ and $\pr[y\neq f_{\vec v}(\x)|\x]=\eta$.
       Let $A_{\vec v}$ be the conditional distribution of $D_{\vec v}$ given $y=1$ and let $B_{\vec v}$ be the conditional distribution of $D_{\vec v}$ given $y=0$. 
       We denote by $A_{\vec v}(\x)$ and $B_{\vec v}(\x)$ the pmf of $A_{\vec v}$ and $B_{\vec v}$, respectively. Moreover, we denote by $\mathcal{U}_d(\x)$ the pmf of $\mathcal{U}_d$. We first give a closed form expression for the pmf of $A_{\vec v}$ and $B_{\vec v}$. Its proof can be found in \Cref{app:lower-bound}.
       \begin{restatable}{claim}{clmPmfConditional}\label{clm:pmf-conditional}
        It holds $A_{\vec v}(\x)=\frac{\eta +(1-2\eta)f_{\vec v}(\x)}{\eta +(1-2\eta)\E[f_{\vec v}(\x)]}\mathcal{U}_d(\x)$ and $B_{\vec v}(\x)=\frac{1-\eta -(1-2\eta)f_{\vec v}(\x)}{1-\eta -(1-2\eta)\E[f_{\vec v}(\x)]}\mathcal{U}_d(\x)$.
       \end{restatable}
Fix $\vec v,\vec u\in\{\pm 1\}^d$. We associate each $\vec v$ and $\vec u$ to a distribution $D_{\vec v}$ and $D_{\vec u}$, constructed as above. The following lemma provides explicit bounds on the correlation between the distributions $D_{\vec v}$ and $D_{\vec u}$, and its proof can be found in \Cref{app:lower-bound}.

Recall that the pairwise correlation of two distributions with pmfs
$D_1, D_2$ with respect to a distribution with pmf $D$ 
is defined as $\chi_{D}(D_1, D_2) + 1 \eqdef \sum_{x\in\mathcal{X} } D_1(x) D_2(x)/D(x)$ (see \Cref{def:pc}). We have the following lemma:

\begin{restatable}{lemma}{lemABCorrelation}\label{lem:a-b-correlation}
Let $D_0$ be a product distribution over 
$\mathcal{U}_d\times \{0,1\}$, 
where $\pr_{(\x,y)\sim D_0}[y=1]=\pr_{(\x,y)\sim D_{\vec v}}[y=1]$.  
We have $\chi_{D_0}(D_{\vec v},D_{\vec u})\leq 2(1-2\eta)(\E[f_{\vec v}(\x)f_{\vec u}(\x)]-\E[f_{\vec v}(\x)]\E[f_{\vec u}(\x)])$ and  
$\chi^2(D_{\vec v},D_0)\leq (1-2\eta)(\E[f_{\vec v}(\x)]-\E[f_{\vec v}(\x)]^2)$.
\end{restatable}

To bound the correlation between $f_{\vec v},f_{\vec u}$, we use the following key lemma whose proof can be found in \Cref{ssec:proof-cor}.

\begin{lemma}[Correlation Bound]\label{lemma:correlation-bound}
    Let $\vec v,\vec u\in\{\pm 1\}^d$ and 
    $f_{\vec v}(\x)=\1\{\vec v\cdot \x \geq 2t-d\}$. 
    Choose $t$ so that $\E[f_{\vec v}(\x)]=\eps$ for $\eps \in (0,1)$. Assume that $|\vec v\cdot \vec u|\leq O(d/\polylog(d/\eps))$. Then there is an absolute constant $C>0$ such that 
$
   \E[ f_{\vec v}(\x)f_{\vec u}(\x)]\leq C\log^2(d/\eps)\eps^2 |\vec v\cdot \vec u|/d+\eps^2\;.
$
\end{lemma}

The following fact states that there exists a large set of almost orthogonal vectors in $\{\pm 1\}^d$. Its proof can be found in \Cref{app:lower-bound}.

\begin{restatable}{fact}{factNearOrthVecDisc}\label{fact:near-orth-vec-disc}
Let $d \in \Z_+$. Let $0<c<1/2$.
There exists a collection $\cal{S}$ of $2^{\Omega(d^c)}$ vectors in $\{\pm 1\}^d$,
such that any pair $\vec v, \vec u\in\cal{S}$, with $\vec v \neq \vec u$, satisfies $|\vec v\cdot \vec u|<d^{1/2+c}$.
\end{restatable}

By \Cref{lemma:correlation-bound}, we get that for any two 
vectors $\vec v,\vec u\in\{\pm 1\}^d$, we have that 
$\chi_{D_0}(D_{\vec v},D_{\vec u})\leq C\log^2(d/\eps)(1-2\eta)\eps^2|\vec v\cdot \vec u|/d$ and  $\chi^2(D_{\vec v},D_0)\leq (1-2\eta)4\eps$ for some $C>0$.

By \Cref{fact:near-orth-vec-disc}, for any $0<c<1/2$, 
there exists a set $S$ of $2^{\Omega(d^c)}$ vectors 
such that for any two vectors $\vec v,\vec u\in S$, 
we have that $|\vec v\cdot \vec u|/d\leq d^{c-1/2}$. 
Denote by $\mathcal{D}$ the set containing the distributions 
$D_{\vec v}$ for each $\vec v\in S$ and let $D_0$ 
be a product distribution over $\mathcal{U}_d\times \{0,1\}$, where $\pr_{(\x,y)\sim D_0}[y=1]=\eta+2(1-2\eta)\eps$. 
By standard results (see \Cref{lem:sq-from-pairwise}),
for the decision problem $\mathcal{B}(\mathcal{D},D_0)$ 
of distinguishing between a distribution in $\mathcal{D}$ 
and the reference distribution $D_0$, the following holds:
any SQ algorithm  either requires a query of tolerance at most 
$O(\eps d^{c/2-1/4})$ or needs to make at least $2^{\Omega(d^c)}$ many queries.

It remains to reduce the testing (decision) problem above to the 
learning problem. 
This is standard, but we include it here for completeness. 
Suppose we have 
access to an algorithm $\cal A$ that solves the 
RCN problem with margin $\gamma$ to excess error 
$\eps'>0$. For the distributions in the set $\mathcal{D}$ of hard distributions,
the margin $\gamma$ is $1/(2d)$. We describe how algorithm  
$\cal A$ can be used to solve the testing problem 
$\mathcal{B}(\mathcal{D},D_0)$. If the underlying distribution were  
$D_{\vec v}$ for some $\vec v\in\{\pm 1\}^d$, then 
algorithm $\cal A$ would produce a hypothesis $h$ such that 
$\pr_{(\x,y)\sim D_{\vec v}}[h(\x)\neq y]\leq \eta+\eps'$. If 
the underlying distribution were $D_0$ --- i.e., the one with 
independent labels --- then the best attainable error would be  
$\eta+2(1-2\eta)\eps$ (achieved by the constant hypothesis $h(\x) \equiv 1$). 
Therefore, for $\eta=1/3$ and $\eps'=\eps/4$, algorithm 
$\cal A$ solves the decision problem $\mathcal{B}(\mathcal{D},D_0)$.
This completes the proof of \Cref{thm:sq-theorem}.

\subsection{Proof of \Cref{lemma:correlation-bound}}\label{ssec:proof-cor}

We start with some definitions of the Fourier transform over the uniform distribution on the hypercube.
For a subset $T\subseteq[d]$ and $\bx \in \{\pm 1\}^d$, we denote $\chi_T(\bx) = {\prod_{i \in T}\x_i}$.
For a  function $f$ from $\{\pm 1\}^d$, let $\wh{f}(T) = \E[f(\vec x)\chi_T(\vec x)]$. For a boolean function $f:\{\pm 1\}^d\mapsto \{0,1\}$, we can write $f$ in the Fourier basis as follows, $f(\x)=\sum_{T\subseteq [d]} \E[f(\vec z)\chi_T(\vec z)]\chi_T(\x)=\sum_{T\subseteq [d]} \wh{f}(T)\chi_T(\x)$. Note that $\chi_T(\x)$ is an orthonormal polynomial basis under the uniform distribution over $\{\pm 1\}^d$; this means that $\E[\chi_T(\x)\chi_{T'}(\x)]= \delta_{T,T'}$, where $ \delta$ is the Kronecker delta. Further, for any two functions $f_1,f_2:\{\pm 1\}^d\mapsto \{0,1\}$, we have that $\E[f_1(\x)f_2(\x)]=\sum_{T\subseteq [d]} \wh{f_1}(T)\wh{f_2}(T)$. We also define the normalized Kravchuk polynomials as follows.

\begin{definition}[Normalized Kravchuk Polynomials] \label{def:KravSym}
For $n,a,b \in \Z_+$ with $0\le a,b\le n$, 
the normalized Kravchuk polynomial $\mathcal{K}(n,a,b)$ 
is defined by 
\[\mathcal{K}(n,a,b):= \frac{1}{\binom{n}{a}\binom{n}{b}}\sum_{A\subseteq[n],B\subseteq[n],|A|=a,|B|=b} (-1)^{|A\cap B|}\;.
\]
\end{definition}
One can think of the normalized Kravchuk polynomial $\mathcal{K}(n,a,b)$
as the expectation over the random subsets $A,B$ of size $a$ and $b$ 
of $-1$ to the number of elements in the intersection of $A$ and $B$. 
Note that for $n,a,b \in \Z_+$ $\mathcal{K}(n,a,b)=\mathcal{K}(n,b,a)$ 
and $|\mathcal{K}(n,a,b)|=|\mathcal{K}(n,a,n-b)|$.
Furthermore, by definition, it also holds that $|\mathcal{K}(n,a,b)|\leq 1$.
The proof of the following lemma can be found in \Cref{app:lower-bound}.
\begin{restatable}{lemma}{lemPropertiesKravchuk}\label{lem:propertiesKravchuk}
    Let $d,m,k\in \Z$. Then the following hold:
    \begin{enumerate}
        \item For $k\leq d/2$, it holds  
     $ \left|\mathcal{K}(d,m,k)\right|\leq e^k2^{3k}\big((kd^{-1})^{k/2}+(|d/2-m|d^{-1})^k\big)\;.
        $
         \item If $k\leq d/2$ and $|d/2-m|\leq d/4$, then  $|\mathcal{K}(d,m,k)|= \exp(-\Omega(k))$.
    \end{enumerate}
\end{restatable}

For a vector $\vec v\in\{\pm 1\}^d$, we define the boolean function $f_{\vec v}(\x)=\1\{\vec v\cdot \x \geq t\}$. We first calculate the Fourier transform of $f_{\vec v}(\x)$. The
proof can be found in \Cref{app:lower-bound}.

\begin{restatable}[Fourier Coefficients]{claim}{claimFourierCoef}\label{claim:fourier-coef}
    Fix vector $\vec v\in\{\pm 1\}^d$ and let $f_{\vec v}(\x)=\1\{\vec v\cdot \x \geq t\}$. 
    For $T\subseteq[d]$, we have that the Fourier coefficient of $f$ at $\chi_{T}(\x)$, 
    i.e., $\wh{f}(T)$, is given by  
 $$\wh{f}(T)=\E[f_{\vec v}(\x)\chi_T(\x)]=\chi_T(\vec v)(-1)^{|T|}2^{-d}\sum_{s=t}^d \binom{d}{s}\mathcal{K}(d,s,|T|)\;.
    $$
\end{restatable}

\noindent\textbf{Proof of \Cref{lemma:correlation-bound}}\quad 
Using \Cref{claim:fourier-coef}, we have that
    \begin{align*}
        \E[ f_{\vec v}(\x)f_{\vec u}(\x)]&= \sum_{T\subseteq[d]}\widehat{f}_{\vec v}(T)\widehat{f}_{\vec u}(T)
        =\sum_{k=0}^d\left(2^{-d}\sum_{s=t}^d\binom{d}{s}\mathcal K(d,k,s)\right)^2\sum_{T\subseteq[d],|T|=k}\chi_T(\vec v)\chi_T(\vec u)
        \\&=\sum_{k=0}^d\binom{d}{k}\left(2^{-d}\sum_{s=t}^d\binom{d}{s}\mathcal K(d,k,s)\right)^2\mathcal K(d,m,k)\;,
    \end{align*}
    where $m$ is the number of components for which $\vec v,\vec u$ agree. We proceed by bounding each term of this sum. To this end, we denote $R_{k}=\binom{d}{k}\left(2^{-d}\sum_{s=t}^d\binom{d}{s}\mathcal K(d,k,s)\right)^2 \mathcal{K}(d,k,m)$. First note that $R_0=\E[f_{\vec v}(\x)]^2$; to see this, observe that $R_0=\E[f_{\vec v}(\x)\chi_{\emptyset}(\x)]^2=\E[f_{\vec v}(\x)]^2$. Next, we bound $R_d$. We have the following claim, whose proof can be found in \Cref{app:lower-bound}.
    \begin{restatable}{claim}{clmLowerBoundEdge}\label{clm:lower-bound-edge}
        It holds that $|R_d|\leq 2^{-2d}\binom{d-1}{t-1}^2$.
    \end{restatable}
Therefore, using that $\binom{d-1}{t-1}=\binom{d-1}{d-t-2}=(t/d)\binom{d}{t}$, we get that $|R_d|\leq (t/d)^2 (2^{-d}\binom{d}{t})^2\leq \eps^2 /d$. 
    We next bound $R_k$ for $k\in \{1,2, \dots, d-1\}$ (see \Cref{app:lower-bound} for the proof).
\begin{restatable}{claim}{clmLowerBoundLargeK} \label{clm:lower-bound-large-k}
    Let $c>0$ be a sufficiently large constant. We have that $ \sum_{k=c\log(d/\eps)}^{d-c\log(d/\eps)} R_k\leq \eps^2/d$.
\end{restatable}
Finally, the 
 following claim bounds the small degree terms.
\begin{claim}\label{clm:lower-bound-small-k}
    Let $k'=c\log(d/\eps)$,  where $c>0$ is the absolute constant as in \Cref{clm:lower-bound-large-k}. For $0\leq k\leq k'$ or $d-k'\leq k\leq d$, we have
    $|R_k|\leq 4\eps^2k'|\vec v\cdot \vec u|/d$.
\end{claim}
\begin{proof}
We provide the proof for the case where $0\leq k\leq k'$, as the other case is symmetric.
    Let $a=|d-2m|/d$ and note that $a=|\vec v\cdot \vec u|/d$. From \Cref{lem:propertiesKravchuk}, we have that  $|\mathcal{K}(d,m,k)|\leq (a^k+(\log(d/\eps)/d)^{k/2})$. Thus, it follows that
\begin{align*}
    |R_k|&=\binom{d}{k}\left(2^{-d}\sum_{s=t}^d\binom{d}{s}\mathcal K(d,k,s)\right)^2\mathcal K(d,m,k)
    \\&\leq 2^{4k}\binom{d}{k}\left(2^{-d}\sum_{s=t}^d\binom{d}{s}\mathcal K(d,k,s)\right)^2(a^k+(\log(d/\eps)/d)^{k/2})
    \\&\leq 2^{4k}\binom{d}{k}\Bigg(2^{-d}\sum_{\substack{s=t\\|s-d/2|\leq c'\sqrt{d k\log(d/\eps)}}}^d\binom{d}{s}|\mathcal K(d,k,s)| + (\eps/d)^{2ck}\Bigg)^2(a^k+(\log(d/\eps)/d)^{k/2})\;,
\end{align*}
where we used that $ |\mathcal K(d,k,s)| \leq 1$ and that $\sum_{i=k}^d\binom{d}{i}2^{-d}\leq 2\exp(-(k-n/2)^2/n)$ from Hoeffding's inequality, hence  $\sum_{\substack{s\geq d/2 c'\sqrt{d k\log(d/\eps)}}}^d\binom{d}{s} \leq (\eps/d)^{2c k}$. Futhermore, note that from \Cref{lem:propertiesKravchuk}, we have that $|\mathcal K(d,k,s)|\leq (2k\sqrt{\log(d/\eps)/d})^k$ for $|s-d/2|\leq c'\sqrt{d k\log(d/\eps)}$. Therefore, we have that
\begin{align*}
    |R_k|&\leq 2^{4k}\binom{d}{k}\left((2\sqrt{\log(d/\eps)/d}k)^k\sum_{s=t}^d\binom{d}{s}2^{-d}+ (\eps/d)^{2ck}\right)^2(a^k+(\log(d/\eps)/d)^{k/2})
    \\&\leq 2^{4k}\eps^2\binom{d}{k}(2\sqrt{\log(d/\eps)/d}k)^{2k}(a^k+(\log(d/\eps)/d)^{k/2})\;,
\end{align*}
where we used that by our choice of $t$ it holds $\E[f_{\vec v}(\x)]=\eps$; therefore, $\sum_{s=t}^d\binom{d}{s}2^{-d}\leq \eps$. Using the fact that $\kappa\leq c \log(d/\eps)$ and that $\binom{d}{k}\leq d^k$, we get that
$|R_k|\leq \eps^2 (2k')^{C'k}(a^{k}+d^{-k/2})\;.
$
Therefore, if $a\leq C/\poly(k')$ for some sufficiently small absolute constant  $C>0$, we get that all the terms are bounded by the first term, i.e., we get that
$|R_{k}|\leq 4\eps^2 k'\alpha$.
\end{proof}
In summary, we have that $\sum_{k=0}^d R_k\leq C\log^2(d/\eps)\eps^2 |\vec v\cdot \vec u|/d+\eps^2$, for some absolute constant $C>0$. 
This completes the proof. \hfill$\blacksquare$

\vspace{-0.2cm}

\section{Conclusions} \label{sec:conc}
We studied the classical problem of learning margin halfspaces
with Random Classification Noise. Our main finding is an information-computation tradeoff
for SQ algorithms and low-degree polynomial tests. Specifically, our lower bounds 
suggest that efficient learners require sample complexity at least $\Omega(1/(\gamma^{1/2} \eps^2))$ (while $\widetilde{O}(1/(\gamma^2 \eps))$ samples information-theoretically suffice).
A number of interesting open questions remain. First, there is still a gap between 
$\widetilde{O}(1/(\gamma^2 \eps^2))$ --- the sample complexity of our algorithm --- and the
lower bound of $\Omega(1/(\gamma^{1/2} \eps^2))$. We believe that an SQ lower bound of 
$\Omega(1/(\gamma \eps^2))$ can be obtained with a more careful construction, but it is not clear what 
the optimal bound may be. Second, it would be interesting to obtain reduction-based computational 
hardness matching our SQ lower bound, along the lines of recent results~\citep{GupteVV22, DKMR22, DKR23}.

\bibliography{clean2}

\appendix
\newpage

\section{Omitted Proofs from \Cref{sec:algo}}\label{appx:omitted-pfs-algo}

\subsection{Proof of \Cref{lemma:main-regret-bnd}}
We restate and prove the following.

\mainregretbnd*
\begin{proof}
We first argue, using standard regret analysis provided for completeness, that
    \begin{equation}\label{eq:SubGD-regret}
        \frac{1}{t+1}\sum_{s=0}^t\hgn(\w_s)\cdot (\w_s - \wstar) \leq \frac{2}{\mu (t+1)} + \frac{\mu (1 - \eta)^2}{2},
    \end{equation}  
where $\mu$ is the step size specified in \Cref{algo:main}. 
    
Fix any $t \in \{0, 1, \dots, T-1\}$. Recall that $\w_{t+1} = \proj_{\mathcal{B}}\big(\w_t - \mu \hat{\vec g}_N(\w_t)\big)$ and $\w^* = \proj_{\cal B}(\w^*).$ Hence, by the nonexpansivity of the projection operator, we have 
\begin{equation}\label{eq:descent-1}
        \|\w_{t+1} - \wstar\|_2 \leq \|\w_t - \wstar - \mu \hat{\vec g}_N(\w_t)\|_2.
    \end{equation}
Further, expanding the square $\|\w_t - \w^* - \mu \hat{\vec g}_N(\w_t)\|_2^2$ and using \Cref{claim:subgrad-bnd}, we get
\begin{align*}
        \|\w_t - \wstar - \mu \hat{\vec g}_N(\w_t)\|_2^2 &= \|\w_t - \wstar\|_2^2 + \mu^2 \|\hat{\vec g}_N(\w_t)\|_2^2 -2\mu\hat{\vec g}_N(\w_t)\cdot( \w_t - \w^*  )\\
        &\leq \|\w_t - \wstar\|_2^2 + \mu^2(1-\eta)^2 -2\mu \hat{\vec g}_N(\w_t)\cdot( \w_t - \w^*  ).
    \end{align*}
Hence, combining the last inequality with \Cref{eq:descent-1}, we get
\begin{equation}\label{eq:pot-change}
        \|\w_{t+1} - \wstar\|_2^2 
        \leq \|\w_t - \wstar\|_2^2 + \mu^2(1-\eta)^2 -2\mu\hat{\vec g}_N(\w_t)\cdot( \w_t - \w^* ).
    \end{equation}
To obtain \Cref{eq:SubGD-regret}, it remains to rearrange \Cref{eq:pot-change} and telescope. In particular, for $T$ itearations and $\mu = \frac{2}{(1-\eta)\sqrt{T+1}},$ we have 
\begin{equation}\label{eq:SubGD-regret-with-mu}
        \frac{1}{T+1}\sum_{t=0}^T\hgn(\w_t)\cdot( \w_t - \wstar) \leq \frac{2(1-\eta)}{\sqrt{T+1}}.
    \end{equation}  
Writing $\hgn(\w_t)\cdot( \w_t - \wstar)$ as $\hgn(\w_t)\cdot( \w_t - \wstar) = \hgn(\wstar)\cdot( \w_t - \wstar) + (\hgn(\w_t) - \hgn(\wstar))\cdot( \w_t - \wstar)$ and rearranging \Cref{eq:SubGD-regret-with-mu}, we further get
\begin{align}
        \frac{1}{T+1}\sum_{t=0}^T(\hgn(\w_t) - \hgn(\wstar))\cdot( \w_t - \wstar) &\leq \frac{2(1-\eta)}{\sqrt{T+1}} + \hgn(\wstar)\cdot\Big(\wstar - \frac{1}{T+1}\sum_{t=1}^t \w_t \Big)\notag\\
        &\leq \frac{2(1-\eta)}{\sqrt{T+1}} + 2\|\hgn(\wstar)\|_2, \label{eq:in-between-regret}
    \end{align}
where we used Cauchy-Schwarz inequality and $\w^*, \w_t \in \mathcal{B},$ for all $t \in \{0, \dots, T\}.$ 

    Recall from \Cref{eq:grad-to-misclassification} that
\begin{align*}
        (\g_{\eta}(\w) - \g_{\eta}(\wstar))\cdot( \w - \wstar) &= (1 - 2\eta) \1\{\sign(\w\cdot \x) \neq \sign(\wstar \cdot \x)\}(|\w\cdot \x| + |\wstar \cdot \x|)\\
        &\geq (1 - 2\eta)\gamma \1\{\sign(\w\cdot \x) \neq \sign(\wstar \cdot \x)\},
    \end{align*}
where the inequality holds by \Cref{assmpt:margin+RCN}, \Cref{assmpt:margin}. Hence, by the definitions of $\hgn$ and $\widehat{\pr}_N$ (from \Cref{eq:emp-grad} and \Cref{eq:emp-prob-disagreement-def}), we have
\begin{equation}\label{eq:emp-class-error-to-regret}
        (\hgn(\w_t) - \hgn(\wstar))\cdot( \w_t - \wstar) \geq (1 - 2\eta)\gamma \widehat{\pr}_N(\w_t).
    \end{equation}
Combining \Cref{eq:in-between-regret} and \Cref{eq:emp-class-error-to-regret}, we then obtain the claimed regret bound, using simple algebraic manipulations. 
\end{proof}

\subsection{Proof of \Cref{lemma:sample-complexity}}

We restate the lemma and provide proof.
\lemmaSampleComplexity*

We first bound the error $E_2 =  \frac{2}{(1-2\eta)\gamma}\|\hgn(\wstar)\|_2$.  Observe that $\E[\g_{\eta}(\wstar)]=0$ and that, by \Cref{claim:subgrad-bnd}, $\|\g_{\eta}(\wstar)\|_2\leq 1-\eta$ surely. We use the following Bennett-type inequality \begin{fact}[\citep{SZ07}, Lemma 1]\label{fct:Bennet-vectors}
    Let $\vec Z_1,\ldots,\vec Z_n\in \R^d$ be random variables such that for each $i\in[n]$ it holds $\|\vec Z_i\|_2\leq M < \infty$ almost surely and let $\sigma^2=\sum_{i=1}^n\E[\|\vec Z_i\|_2^2]$. Then, we have that for any $\eps > 0,$
    \[
    \pr\left[\left\|\frac{1}{n}\sum_{i=1}^n \left(\vec Z_i- \E[\vec Z_i]\right)\right\|_2\geq \eps\right]\leq 2\exp\left(-\frac{n\eps}{2M}\log\bigg(1+\frac{n M\eps}{\sigma^2}\bigg)\right)\;.
    \]
    \end{fact}
    Note that $\E[\|\g_{\eta}(\wstar)\|_2^2=O(\eta)$ and using \Cref{fct:Bennet-vectors}, along with the inequality $\log(1+z)\geq z/2$, for $z\in(0,1)$ (note that $\sigma^2$ is at most $nM$), we get that for any $\hat{\epsilon}$ and $N\geq \Omega(\frac{\log(1/\delta)}{\hat{\eps}^2})$, with probability at least $1-\delta/2$, we have 
\[
\|\E[\g_{\eta}(\wstar)]-\hgn(\wstar)\|_2 = \|\hgn(\wstar)\|_2\leq \hat{\eps}\;.
\]
To complete bounding $E_2,$ it remains to choose $\hat{\epsilon} = \frac{(1-2\eta)\gamma\bar{\epsilon}}{8}$.

    To complete the proof and bound 
    $E_3 = \frac{1}{T+1}\sum_{t=0}^T\big[\pr[\sign(\w_t\cdot \x) \neq \sign(\wstar \cdot \x)] - \widehat{\pr}_N(\w_t)\big]$, 
    we use uniform convergence results for the function 
    $\w \mapsto \1\{\sign(\w\cdot \x) \neq \sign(\wstar \cdot \x)\}$. 
    Note that this boolean concept class is a subset of the class 
    of the intersection of two halfspaces. The latter class has 
    VC dimension $O(d)$. Thus, by the standard VC inequality 
    combined with uniform convergence (see e.g., p. 31 of~\citet{DL:01}) we have that 
    $N = {\Omega}(\frac{d}{\bar{\epsilon}}\log(1/\delta))$ 
    samples suffice so that with probability $1 - \delta/2,$ 
    we have $E_3 \leq \frac{\bar{\epsilon}}{4}.$
    \hfill$\blacksquare$

\subsection{Testing to Find the Right Hypothesis}

The following lemma justifies the claim made at the end of the proof of \Cref{thm:main-algo} and completes its proof. We use the following fact from \citet{SB14}.
\begin{fact}[Thereom 6.8 of \citet{SB14}]\label{lemma:testing}
    Given a finite set of hypotheses $\mathcal H$, by drawing $N=O(\frac{\log(|\mathcal H|+\log(\frac{1}{\delta}))}{\eps^2})$ samples, it is guranteed that with probability at least $1-\delta$ it holds
\[
\min_{h\in \mathcal H}\Big\{ \frac{1}{N}\sum_{i=1}^N \1\{h(\x^{(i)})\neq y^{(i)}\}\Big\} \leq \min_{h\in \mathcal H}  \pr[h(\x)\neq y] +\eps\;.
\]
\end{fact}

\subsection{Dimension Reduction}

We will use the following Johnson-Lindenstrauss lemma as our main technique to reduce the dimension of the space.

\begin{lemma}[Johnson-Lindenstrauss]\label{lem:JL}
	Let $\beta,\eps$ be some positive constants. Let $A\in\R^{m\times d}$ be a random matrix with each entry $A_{ij}$ sampled from $\mathrm{Uniform}\{-1/\sqrt{m}, 1/\sqrt{m}\}$ where $m= O(\log(1/\beta)/\eps^2)$. Then, for any unit vector $\bu, \bv\in\R^d$, it holds
	\begin{equation*}
		\pr_A[|\bu\cdot\bv - (A\bu)\cdot(A\bv)|\geq \eps]\leq \beta.
	\end{equation*}
 Consequently, for any unit vector $\bu\in\R^d$, we have $\pr_A[|\|\bu\|_2^2 - \|A\bu\|_2^2|\geq \eps]\leq \beta$.
\end{lemma}

\begin{corollary}\label{cor:JL}
    Let $\beta, \beta', \gamma$ be some positive constants such that $\beta\leq \beta'$ and $\gamma\in(0,1)$. Let $A\in\R^{m\times d}$ be a random matrix with each entry $A_{ij}$ sampled uniformly at random from $\{-1/\sqrt{m}, 1/\sqrt{m}\}$ where $m=O(\log(1/\beta)/\gamma^2)$. Then, for $\x\sim\D_\x$ and $\w^*\in\mathcal{B}$, with probability at least $1-\beta/\beta'$ it holds
    \begin{equation*}
        \pr_{\x\sim\D_\x}[|\w^*\cdot\x - (A\w^*)\cdot(A\x)|\geq \gamma/2]\leq \beta'.
    \end{equation*}
    In addition, with probability at least $1-\beta/\beta'$ it holds $\pr_{\x\sim\D_\x}[|\|\x\|_2^2 - \|A\x\|_2^2|\geq \gamma/2]\leq \beta'$.
\end{corollary}
\begin{proof}
    \Cref{lem:JL} indicates that since the matrix $A$ is generated independent of the distribution $\D_\x$, for any unit vector $\x\sim\D_\x$ the norm of the transformed vector $A\x$ is close to 1 with constant probability. To be specific, we have:
    \begin{equation}\label{eq:full-expectation of wx-AxAw}
        \pr_{A,\x\sim\D_\x}[|\w^*\cdot\x - (A\w^*)\cdot(A\x)|\geq \gamma/2]\leq \beta.
    \end{equation}
    Now let $P(A) = \pr_{\x\sim\D_\x}[|\w^*\cdot\x - (A\w^*)\cdot(A\x)|\geq \gamma/2] = \E_{\x\sim\D_\x}[\1\{|\w^*\cdot\x - (A\w^*)\cdot(A\x)|\geq\gamma/2\}|A]$ be a random variable determined by $A$. Note that $\E_A[P(A)]\leq \beta$ by \Cref{eq:full-expectation of wx-AxAw}. Thus, applying Markov inequality to $P(A)$ we get
    \begin{equation*}
        \pr_A[P(A) \geq \beta']\leq \frac{\E_A[P(A)]}{\beta'}\leq \frac{\beta}{\beta'}.
    \end{equation*}
    Therefore, for any given matrix $A$ sampled from the distribution $A_{ij}\sim\mathrm{Uniform}\{\pm\frac{1}{\sqrt{m}}\}$ where $m = O(\log(1/\beta)/\gamma^2)$, we have with probability at least $1 - \beta/\beta'$,
    \begin{equation*}
        \pr_{\x\sim\D_\x}[|\w^*\cdot\x - (A\w^*)\cdot(A\x)|\geq \gamma/2]\leq \beta'.
    \end{equation*}
    Following the same idea, we can also show that with probability at least $1-\beta/\beta'$ it holds $\pr_{\x\sim\D_\x}[|\|\x\|_2^2 - \|A\x\|_2^2|\geq \gamma/2]\leq \beta'$.
\end{proof}

\begin{algorithm}[ht!]

\KwIn{ $\epsilon > 0$, $\gamma \in(0,1)$, $\eta \in (0, 1/2),$ $m > 0$, sample $\{\x^{(i)}, y^{(i)}\}_{i=1}^N$ drawn i.i.d.\ from $\D$, a random matrix $A\in\R^{m\times d}$ generated such that $A_{ij} = 1/\sqrt{m}$ w.p. $1/2$ and $A_{ij} = -1/\sqrt{m}$ w.p. $1/2$, any $\vec w_0 \in \mathcal{B}$} 
$T = \big\lceil \big(\frac{48(1-\eta)}{(1-2\eta)\gamma\eps}\big)^2 - 1 \big\rceil,$ $\mu = \frac{1}{(1-\eta)\sqrt{T+1}}$, $\Bar{\x}\ith = A\x\ith$ for $i = 1,\cdots, N$.

\For{$t = 0:T-1$}
{
$\Bar{\vec g}_N(\Bar{\w}_t) = \frac{1}{2N}\sum_{i=1}^N \big((1-2\eta)\sign(\Bar{\w}_t \cdot \bar{\x}^{(i)}) - y^{(i)}\big)\Bar{\x}^{(i)}$

$\Bar{\w}_{t+1} = \proj_{\|\Bar{\w}\|_2\leq 1}\big(\Bar{\w}_t - \mu \Bar{\vec g}_N(\Bar{\w}_t)\big)$

}
\KwRet{ $\{A^\top\Bar{\w}_0, \dots, A^\top\Bar{\w}_T\}$} 

\caption{Dimension-Reduced Margin Halfspace Learner with RCN}\label{algo:main-dim-reduced}

\end{algorithm}

\begin{theorem}
    Fix $\eps>0,\gamma\in(0,1)$. Let the number of iterations be $T = O\big(\frac{(1-\eta)^2}{\gamma^2\eps^2}\big)$ and set the stepsize $\mu = \frac{1}{(1-\eta)\sqrt{T+1}}$. Furthermore, let $A\in\R^{m\times d}$ be a matrix generated from the distribution described in \Cref{algo:main-dim-reduced} with $m = O\big(\frac{\log((1-2\eta)/(\delta\eps))}{\gamma^2}\big)$. Then running \Cref{algo:main-dim-reduced} for $T$ iterations with $N = \wt{\Omega}\big( \big(\frac{\eta}{\gamma^2\eps^2} + \frac{(1-2\eta)}{\gamma^2\eps}\log\big(\frac{1-2\eta}{\delta\eps}\big)\big) \log(\frac{1}{\delta}) \big)$ i.i.d.\ samples drawn from distribution $\D$, \Cref{algo:main-dim-reduced} learns $\gamma$-margin halfspaces corrupted with $\eta$-RCN with error $\eta + \eps$ with probability at least $1-\delta$.
\end{theorem}

\begin{proof}
    The goal of the proof is to show that the analysis of \Cref{algo:main} can be transformed to \Cref{algo:main-dim-reduced} with minor modifications.
    For simplicity, let's denote $\barw = A\w\in\R^m$ and $\barx = A\x\in\R^m$ for any $\w,\x\in\R^d$. Similarly, we have $\barw^* = A\w^*$ and $\barx\ith = A\x\ith$.

    We first show that as a consequence of \Cref{lem:JL} and \Cref{cor:JL}, with a large probability that the $A$ generated in \Cref{algo:main-dim-reduced} is a ``good matrix'' in the sense that for most of the points in $\D$, $A\x$ will not be far away from $\x$. Formally, we have the following claim.
    
    \begin{claim}\label{claim:condition-on-A}
        Fix some constants $\gamma,\bar{\eps},\delta > 0, N > 1$ and let $m = O(\frac{\log(1/\beta)}{\gamma^2})$ where $\beta = \frac{\bar{\eps}\delta}{20N}$. For any $A$ generated in \Cref{algo:main-dim-reduced}, denote $\mathcal{E}_A = \{\x\in\mathcal{S}^{d-1}: |\w^*\cdot\x - \barw^*\cdot\barx|\leq \gamma/2,\; |\|\x\|_2^2 - \|\barx\|_2^2|\leq \gamma/2\}$ and let $\evt$ be the set of $A$ such that $\barw^*$ is close to $\w^*$ and moreover, for any $\x\sim\D_\x$, $\x\in\evt_A$ with high probability, i.e., $\mathcal{E} = \big\{A\in\R^{m\times d}: \pr_{\x\sim\D_\x}[\x\in\mathcal{E}_A]\geq 1 - \bar{\eps}/(2N),\;|\|\w^*\|_2^2 - \|\barw^*\|_2^2|\leq \gamma/2\big\}$. Then, 
        \begin{enumerate}
            \item $\evt$ happens with probability at least $1-\frac{2}{5}\delta$;
            \item If $\evt$ happens, then for any $N$ i.i.d.\ samples $\{\x\ith\}_{i=1}^N$, it holds $\pr[\x\ith \in\evt_A,\forall i\in[N]\,]\geq 1-\frac{\bar{\eps}}{2}$.
        \end{enumerate}
    \end{claim}
    \begin{proof}
     According to \Cref{lem:JL}, we know that $\pr[|\|\w^*\|_2^2 - \|\barw^*\|_2^2|\leq \gamma/2]\geq 1-\beta = 1-\bar{\eps}\delta/(20 N)$. Furthermore, recall that in \Cref{cor:JL} (with a union bound) we showed with probability at least $1-4N\beta/\bar{\eps} = 1 - \delta/5$, $\pr_{\x\sim\D_\x}[\x\in\mathcal{E}_A]\geq 1-\bar{\eps}/(2N)$, therefore the first claim follows from a union bound on these 2 events.
     
     Now conditioned on $\evt$. Given any $N$ samples $\{\x\ith,y\ith\}^N_{i=1}$, we know that $\pr[\x\ith\in\evt_A]\geq 1 - \frac{\bar{\eps}}{2N}$, hence applying union bound we get $\pr[\x\ith\in\evt_A, \forall i\in[N]\,]\geq 1-\frac{\bar{\eps}}{2}$.
\end{proof}

    For the analysis below, we will condition on the event $A\in\evt$ which occurs with probability at least $1 - \frac{2}{5}\delta$. We begin with showing that under such fixed $A$, the norm of $\|\Bar{\g}_N(\barw)\|_2$ can be bounded by $2(1-\eta)$.
    \begin{claim}\label{claim:norm-bar-g}
        Given $N$ samples $\{\x\ith,y\ith\}^N_{i=1}$ and suppose that $\x\ith\in\mathcal{E}_A,\forall i=1,\cdots, N$, we have
        \begin{equation*}
            \sup_{\barw\in\R^m}\|\Bar{\g}_N(\barw)\|_2\leq 2(1-\eta).
        \end{equation*}
    \end{claim}
    \begin{proof}
Since $\x\ith\in\evt_A$, we have $\|\barx\ith\|_2\leq \|\x\ith\|_2 + \gamma/2$. Hence, it holds
        \begin{align*}
             \sup_{\barw\in\R^m}\|\Bar{\g}_N(\barw)\|_2 &= \sup_{\barw\in\R^m}\bigg\{\frac{1}{2N}\big\|\sum_{i=1}^N \big((1-2\eta)\sign(\Bar{\w} \cdot \bar{\x}^{(i)}) - y^{(i)}\big)\Bar{\x}^{(i)}\big\|_2\bigg\}\\
             &\leq (1-\eta)\|\barx\ith\|_2 \leq (1-\eta)(\|\x\ith\|_2 + \gamma/2)\leq 2(1-\eta),
        \end{align*}
        where in the last inequality we used the fact that $\|\x\ith\|_2 = 1$ and $\gamma\in(0,1)$.
    \end{proof}

     We then study the difference between $\|\barw_{t+1} - \barw^*\|_2$ and $\|\barw_t - \barw^*\|_2$, which is almost analogous to the analysis we have seen in \Cref{lemma:main-regret-bnd} with the only differences being that: (i) with probability at least $1-\bar{\eps}/2$ we have $\x\ith\in\evt_A$ for $i\in[N]$, hence $\|\Bar{\g}_N(\barw)\|_2\leq 2(1-\eta)$ for every $\barw_t$ and $\barw^*$ as shown in \Cref{claim:condition-on-A} and \Cref{claim:norm-bar-g}; (ii) since $|\w^*\cdot\x\ith - \barw^*\cdot\barx\ith|\leq \gamma/2$ for all $i\in[N]$, it holds $\sgn(\w^*\cdot\x) = \sgn(\barw^*\cdot\barx\ith)$; (iii) $\|\barw^*\|_2\leq 2$ since we have $|\|\w^*\|_2^2 - \|\barw^*\|_2^2|\leq \gamma/2$ conditioning on the event $\evt$. 
     
     Now we further condition on the event that $\x\ith\in\evt_A$ for $i\in[N]$ and denote the distribution of $\D_\x$ restricted on $\evt_A$ as $\D_\x(\evt_A)$. Then, choosing $\mu = \frac{1}{(1-\eta)\sqrt{T+1}}$ and following the same steps as in \Cref{lemma:main-regret-bnd}, we have
    $
        \pr_{\x\sim\D_\x(\evt_A)}[\sgn((A^\top\barw_t)\cdot\x)\neq \sgn(\w^*\cdot\x)]\leq E_1'+E_2'+E_3',
    $
    where $E_1'=\frac{8(1 - \eta)}{(1 - 2\eta)\gamma\sqrt{T + 1}}$, $E_2'=\frac{6}{(1 - 2\eta)\gamma}\|\Bar{\g}_N(\barw^*)\|_2$ and $E_3' = \frac{1}{T+1}\sum_{t=0}^T\big\{\pr_{\x\sim\D_\x(\evt_A)}[\sgn((A^\top\barw_t)\cdot\x)\neq \sgn(\w^*\cdot\x)] - \wh{\pr}_N(A^\top\barw)\big\}$.

     We show that our choice of $N$ and $T$ suffices to make $\pr_{\x\sim\D_\x}[\sgn((A^\top\barw_t)\cdot\x)\neq \sgn(\w^*\cdot\x)]\leq \bar{\eps}$. First, $T=\big(\frac{48(1-\eta)}{(1-2\eta)\gamma\bar{\eps}}\big)^2$ renders $E_1'\leq \bar{\eps}/6$. Next, observe that $\E_{\x\sim\D_\x(\evt_A)}[\g_\eta(\barw^*;\barx,y)] = 0$ and recall that $\|\g_\eta(\barw^*;\barx\ith,y\ith)\|_2\leq 2(1-\eta)$, thus by \Cref{fct:Bennet-vectors} we know $N= \Omega\big(\log(1/\delta)/\bar{\eps}^2\big)$ suffices to make $E_2'\leq \bar{\eps}/6$ with probability $1-\delta/10$. Finally, since linear threshold function class $\barw\mapsto \1\{\sgn((A^\top\barw)\cdot\x)\neq \sgn(\w^*\cdot\x)\}$ has VC dimension $m+1$, therefore by standard VC dimension arguments choosing $N\geq\Omega(\frac{m}{\bar{\eps}}\log(1/\delta))$ we are guaranteed to have $E_3'\leq \frac{\bar{\eps}}{6}$ with probability $1-\delta/10$. Recall that $\pr[\x\in\evt_A]\geq 1-\bar{\eps}/2$, hence under our choice of $m,N,T$, under the condition of event $\evt$ and $\x\ith\in\evt_A$ for $i\in[N]$, we have with probability at least $1-\delta/5$,
     \begin{equation*}
         \pr_{\x\sim\D_\x}[\sgn((A^\top\barw_t)\cdot\x)\neq \sgn(\w^*\cdot\x)] \leq \pr_{\x\in\D_\x(\evt_A)}[\sgn((A^\top\barw_t)\cdot\x)\neq \sgn(\w^*\cdot\x)]\bigg(1-\frac{\bar{\eps}}{2}\bigg) + \frac{\bar{\eps}}{2}\leq \bar{\eps}.
     \end{equation*}
     Finally, applying a union bound on all of these 3 events, we know that with probability at least $1-\frac{2}{5}\delta - \frac{\bar{\eps}}{2} - \frac{1}{5}\delta\geq 1-\delta$, $\pr_{\x\sim\D_\x}[\sgn((A^\top\barw_t)\cdot\x)\neq \sgn(\w^*\cdot\x)]\leq \bar{\eps}$. By substituting $\bar{\eps}$ with $\eps/(1-2\eta)$, we get a learner with error $\eta + \eps$ and the proof is complete.
    \end{proof}

\section{Omitted Proofs from \Cref{sec:lowerbound}}\label{app:lower-bound}

\subsection{Additional Background on SQ Model}\label{app:sq-background}

We will use the framework of Statistical Query (SQ) algorithms for problems 
over distributions~\citep{FeldmanGRVX17}.
We require the following standard definition.
\begin{definition}[Decision/Testing Problem over Distributions]\label{def:decision}
Let $D$ be a distribution and $\D$ be a family of distributions over $\R^d$. 
We denote by $\mathcal{B}(\mathcal{D},D)$ the decision (or hypothesis testing) 
problem in which the input distribution $D'$ is promised to satisfy either 
(a) $D'=D$ or (b) $D'\in\mathcal{D}$, and the goal of the algorithm 
is to distinguish between these two cases.
\end{definition}

\noindent To define the SQ dimension, we need the following definition.

\begin{definition}[Pairwise Correlation] \label{def:pc}
The pairwise correlation of two distributions with probability mass functions (pmfs)
$D_1, D_2 : \mathcal{X} \to \R_+$ with respect to a distribution with pmf $D:\mathcal{X} \to \R_+$,
where the support of $D$ contains the supports of $D_1$ and $D_2$,
is defined as $\chi_{D}(D_1, D_2) + 1 \eqdef \sum_{x\in\mathcal{X} } D_1(x) D_2(x)/D(x)$.
We say that a collection of $s$ distributions $\mathcal{D} = \{D_1, \ldots , D_s \}$ over $\mathcal{X} $
is $(\gamma, \beta)$-correlated relative to a distribution $D$ if
$|\chi_D(D_i, D_j)| \leq \gamma$ for all $i \neq j$, and $|\chi_D(D_i, D_j)| \leq \beta$ for $i=j$.
\end{definition}

\noindent The following notion of dimension effectively characterizes the difficulty of the decision problem.

\begin{definition}[SQ Dimension] \label{def:sq-dim}
For $\gamma ,\beta> 0$, a decision problem $\mathcal{B}(\mathcal{D},D)$, 
where $D$ is fixed and $\mathcal{D}$ is a family of distributions over $\mathcal{X}$,
let $s$ be the maximum integer such that there exists
$\mathcal{D}_D \subseteq \D$ such that $\D_D$ is $(\gamma,\beta)$-correlated
relative to $D$ and $|\D_D|\ge s$.
We define the {\em Statistical Query dimension} with pairwise correlations $(\gamma, \beta)$
of $\mathcal{B}$ to be $s$ and denote it by $\mathrm{SD}(\mathcal{B},\gamma,\beta)$.
\end{definition}

\noindent The connection between SQ dimension and lower bounds is captured by the following lemma.

\begin{lemma}[\citep{FeldmanGRVX17}] \label{lem:sq-from-pairwise}
Let $\mathcal{B}(\D,D)$ be a decision problem, 
where $D$ is the reference distribution and $\D$ is a class of distributions over $\mathcal{X}$. 
For $\gamma, \beta >0$, let $s= \mathrm{SD}(\mathcal{B}, \gamma, \beta)$.
Any SQ algorithm that solves $\mathcal{B}$ with probability at least $2/3$ 
requires at least $s \cdot \gamma /\beta$ queries to the
$\mathrm{STAT}(\sqrt{2\gamma})$ or $\mathrm{VSTAT}(1/\gamma)$ oracles.
\end{lemma}

\subsection{Proof of \Cref{fact:near-orth-vec-disc}}
We restate and prove the following fact.

\factNearOrthVecDisc*

\begin{proof}
    Sample two vectors $\vec v,\vec u$ at random, i.e., $\vec v,\vec u\sim \mathcal U_d$. Note that $\vec v\cdot \vec u$ is a sum of Rademacher random variables. We use the following concentration inequality:
    \begin{fact}\label{fct:Rademacher} Let $z_1,\ldots, z_n$ be Rademacher random variables. Then, for any $t>0$, it holds
       \[
       \pr\left[\left|\sum_{i=1}^n z_i \right|\geq t\sqrt{n}\right] \leq2  \exp\left(-\frac{t^2}{2}\right).
       \] 
    \end{fact}
    Using \Cref{fct:Rademacher}, we get that for $t=d^{c}$ for some $0<c<1/2$, we get that
    \[
   \pr\left[\left|\vec v\cdot \vec u \right|\geq d^{1/2+c}\right] \leq 2  \exp\left(-\frac{d^{2c}}{2}\right). 
    \]
    From union bound we get that there exists $2^{\Omega(d^c)}$ such vectors.
\end{proof}
\subsection{Proof of \Cref{clm:pmf-conditional}}
       We restate and prove the following claim.

\clmPmfConditional*
       
\begin{proof}
    Denote $g_{D_{\vec v}}$ the pmf of $D_{\vec v}$. We show the claim only for the distribution $A_{\vec v}$ as $B_{\vec v}$ follows similarly. Note that $A_{\vec v}(\x)= \frac{g_{D_{\vec v}}(\x,y=1)}{\pr[ y=1]}$. By construction, we have that $\pr[ y=1]=\eta +(1-2\eta)2\eps$ and $g_{D_{\vec v}}(\x,y=1)=(\eta \1\{f_{\vec v}(\x)=0\}+(1-\eta)\1\{f_{\vec v}(\x)> 0\})\mathcal{U}_d(\x)=(\eta +(1-2\eta)f_{\vec v}(\x))\mathcal{U}_d(\x)$. Therefore, $A_{\vec v}(\x)=\frac{\eta +(1-2\eta)f_{\vec v}(\x)}{\eta +(1-2\eta)2\eps}{\mathcal{U}_d(\x)}$. Similarly, we show that $B(\x)=\frac{1-\eta -(1-2\eta)f_{\vec v}(\x)}{1-\eta -(1-2\eta)2\eps}{\mathcal{U}_d(\x)}$.
\end{proof}

\subsection{Proof of \Cref{lem:a-b-correlation}}
We restate and prove the following.

\lemABCorrelation*

\begin{proof}
Denote $\kappa_1=1/\pr_{(\x,y)\sim D_{\vec v}}[y=1]$ and $\kappa_0=1/\pr_{(\x,y)\sim D_{\vec v}}[y=0]$. We have that
\begin{align*}
    \chi_{D_0}(D_{\vec v},D_{\vec u})&= \pr_{(\x,y)\sim D_{\vec v}}[y=1]\chi_{\mathcal{U}_d}(A_{\vec v},A_{\vec u}) +\pr_{(\x,y)\sim D_{\vec v}}[y=0]\chi_{\mathcal{U}_d}(B_{\vec v},B_{\vec u})
    \\&=\kappa_1^{-1}\chi_{\mathcal{U}_d}(A_{\vec v},A_{\vec u}) +\kappa_0^{-1}\chi_{\mathcal{U}_d}(B_{\vec v},B_{\vec u})\;.
\end{align*} 
We now bound each term in the above expression.
\begin{claim}\label{clm:bound-a-b-terms}
    We have $\chi_{\mathcal{U}_d}(A_{\vec v},A_{\vec u})\leq (1-2\eta)\kappa_1^2(\E[f_{\vec v}(\x)f_{\vec u}(\x)]-\E[f_{\vec v}(\x)]\E[f_{\vec u}(\x)])$ and $\chi_{\mathcal{U}_d}(B_{\vec v},B_{\vec u})\leq (1-2\eta)\kappa_0^2(\E[f_{\vec v}(\x)f_{\vec u}(\x)]-\E[f_{\vec v}(\x)]\E[f_{\vec u}(\x)])$.
\end{claim}

\begin{proof}
We first bound $\chi_{\mathcal{U}_d}(A_{\vec v},A_{\vec u})$ as the other follows similarly.
We have that
\begin{align*}
    \chi_{\mathcal{U}_d}(&A_{\vec v},A_{\vec u})=
    \sum_{\x\in\{\pm 1\}^d}\frac{(A_{\vec v}(\x)-\mathcal{U}_d(\x))(A_{\vec u}(\x)-\mathcal{U}_d(\x))}{\mathcal{U}_d(\x)}\\
    &=\frac{1-2\eta}{(\eta +(1-2\eta)\E[f_{\vec v}(\x)])^2}\sum_{\x\in\{\pm 1\}^d}(f_{\vec v}(\x)-\E[f_{\vec v}(\x)])(f_{\vec u}(\x)-\E[f_{\vec u}(\x)])\mathcal{U}_d(\x)
    \\&=\frac{1-2\eta}{(\eta +(1-2\eta)\E[f_{\vec v}(\x)])^2}\left(\E[f_{\vec v}(\x)f_{\vec u}(\x)]-\E[f_{\vec v}(\x)]\E[f_{\vec u}(\x)]\right)\;.
\end{align*}
Working in a similar way, we also get that $  \chi_{\mathcal{U}_d}(B_{\vec v},B_{\vec u})
   =(1-2\eta)\kappa_0^2(\E[f_{\vec v}(\x)f_{\vec u}(\x)]-\E[f_{\vec v}(\x)]\E[f_{\vec u}(\x)])$.
\end{proof}
Using \Cref{clm:bound-a-b-terms}, we get that
\[
\chi_{D_0}(D_{\vec v},D_{\vec u})\leq (1-2\eta)(\kappa_1+\kappa_0)\E[f_{\vec v}(\x)f_{\vec u}(\x)]\leq 2(1-2\eta)\E[f_{\vec v}(\x)f_{\vec u}(\x)]\;.
\]

 It remains to bound  $ \chi^2(D_{\vec v},D_0)$.
We show the following:
\begin{claim}
   Let $\kappa=(1-2\eta)/(\eta +(1-2\eta)2\E[f_{\vec v}(\x)] -\E[f_{\vec v}(\x)]^2)^2$. It holds that 
    \[
    \chi^2(D_{\vec v},D_0)\leq (1-2\eta)\E[f_{\vec v}(\x)] -\E[f_{\vec v}(\x)]^2\;.
    \] 
\end{claim}
\begin{proof}
Let $\kappa=(1-2\eta)/(\eta +(1-2\eta)2\E[f_{\vec v}(\x)] -\E[f_{\vec v}(\x)]^2)^2$
    We have that
    \[
    \chi^2(D_{\vec v},D_0)=\kappa_1^{-1}\chi^2(\mathcal{U}_d,A_{\vec v})+\kappa_0^{-1}\chi^2(\mathcal{U}_d,B_{\vec v})\;.
    \]
\begin{align*}
    \chi^2(\mathcal{U}_d,A_{\vec v})=
    \sum_{\x\in\{\pm 1\}^d}\frac{(A_{\vec v}(\x)-\mathcal{U}_d(\x))^2}{\mathcal{U}_d(\x)}&=(1-2\eta)\kappa_1^2\sum_{\x\in\{\pm 1\}^d}(f_{\vec v}(\x)-\E[f_{\vec v}(\x)])^2\mathcal{U}_d(\x)
   \\& = (1-2\eta)\kappa_1^2(\E[f_{\vec v}(\x)] -\E[f_{\vec v}(\x)]^2)\;.
\end{align*}
 Similarly, we show that  $\chi^2(\mathcal{U}_d,B_{\vec v})=(1-2\eta)\kappa_0^2(\E[f_{\vec v}(\x)] -\E[f_{\vec v}(\x)]^2)$. Combining, we get the result.
\end{proof}
\end{proof}

\subsection{Proof of \Cref{lem:propertiesKravchuk}}

\new{The following is a more detailed version of \Cref{lem:propertiesKravchuk}. }
\begin{lemma}
    Let $d,m,k\in \Z$. Then, the following hold
    \begin{enumerate}
        \item $|\mathcal{K}(d,m,k)|\leq 1$ for any $k\in \Z$.
        \item For $k\leq d/2$, it holds
        \[
      \left|\mathcal{K}(d,m,k)\right|\leq e^k2^{3k}\left(\left(\frac{k}{d}\right)^{k/2}+\left(\frac{|d/2-m|}{d}\right)^k\right)\;.
        \]
         \item If $k\leq d/2$ and $|d/2-m|\leq d/4$, then  $|\mathcal{K}(d,m,k)|= \exp(-\Omega(k))$. 
    \end{enumerate}
\end{lemma}
\begin{proof} The first part follows from the fact that $\mathcal{K}(d,m,k)$ is the expectation of a random variable with support in $[-1,1]$.
For the next claims, we use the following fact.
\begin{fact}[Claim 22 of \citet{BIJL21}]
    Let $d,m,k\in \Z$, then
    \[
    \left|\mathcal{K}(d,m,k)\right|\leq \frac{e^k2^{3k}}{\binom{d}{k}}\left(\left(\frac{d}{k}\right)^{k/2}+\left(\frac{|d/2-m|}{k}\right)^k\right)\;.
    \]
\end{fact}
Using the inequality $\binom{d}{k}\geq (d/k)^k$, we have that
\[
\left|\mathcal{K}(d,m,k)\right|\leq e^k2^{3k}\left(\left(\frac{k}{d}\right)^{k/2}+\left(\frac{|d/2-m|}{d}\right)^k\right)\;.
\]
If $k\leq d/12$ and $|d/2-m|\leq d/4$, then we have that 
$|\mathcal{K}(d,m,k)|\leq 2\exp(-0.2 k)$.
We provide a proof for the final part, 
i.e., the case where $d/12\leq k\leq d/2$. Denote $Y_{A,B}=(-1)^{A\cap B}$. The sum we want to bound is equal to 
$\E_{A,B}[Y_{A,B}]$. Denote $A'=\{1,3,\ldots,2m-1\}$ and $s_i^B=|B\cap \{2i-1,2i\}|$ for $i=1,\ldots,m$. Note that the $\E_{A,B}[Y_{A,B}\mid A=A',s_1^B,\ldots,s_m^B]=0$, if we condition that any $s_i^B=1$. This holds because if $s_i^B=1$ for some $i$, then we can swap which $2i$ and $2i-1$ is in $B$ to create $B'$ and hence $(-1)^{A\cap B}+(-1)^{A\cap B'}=0$.

It suffices to show that if we choose $B$ at random, i.e., $B$ is a uniform subset of $[d]$ of size $k$, then with probability at most $\exp(-\Omega(k))$
we are in the case where no $s_i^B$ is equal to $1$. To show that, we create a new random variable $B'$ and we sample $B'$ as follows: we let the size of $B'$ be 
equal \new{to} $\mathrm{Bin}(d,k/d)$,  
which is equivalent to sampling each element of $[d]$ with probability $k/d$. Now each $s_i^{B'}$ is independent of each other. The probability \new{that} each $s_i^{B'}$ \new{is} equal to one 
is $2k/d(1-k/d)=\Omega(k/d)$. Therefore, the probability that no $s_i^{B'}$ is one is at most 
$$(1-\Omega(k/d))^m\leq \exp(-\Omega(-km/d))=\exp(-\Omega(k))\;,$$ 
where we used that $|d/2-m|\leq d/4$. 
Therefore, $|\E[Y_{A,B'}|A]|\leq \exp(-\Omega(k))$. 
It remains to \new{relate} the expectation with respect $B'$ 
to the expectation of $B$. Note that according to the sampling rule, 
there is an $\Omega(1/\sqrt{k})$ probability of generating 
a uniform subset of size $k$, but the probability {that} $Y$ 
\new{is} non-zero is at most $\exp(-\Omega(k))$. 
Therefore, we have that 
$$|\E[Y_{A,B'}|A]|=|\E[Y_{A,B'}|A,|B'|=k]|\leq \exp(-\Omega(k))\sqrt{k}=\exp(-\Omega(k)) \;,$$ 
where we used that $k$ is large enough by assumption.
Therefore, $|\E[Y_{A,B}|A]|\leq \exp(-\Omega(k))$ and  
the total expectation is at most $\exp(-\Omega(k))$.

\end{proof}

\subsection{Proof of \Cref{claim:fourier-coef}}
We restate and prove the following.

\claimFourierCoef*

\begin{proof}
    We have that
    \begin{align*}
         \E[f_{\vec v}(\x)\chi_T(\x)]&=2^{-d}\sum_{\x\in\{\pm 1\}^d} f_{\vec v}(\x)\chi_T(\x)
         \\&=2^{-d}\sum_{s=t}^d \sum_{A\subseteq [d],|A|=s}\prod_{i\in T\cap A}\vec v_i \prod_{i\in T\cap \bar A}(-\vec v_i)
         \\&=\chi_T(\vec v)(-1)^{|T|}2^{-d}\sum_{s=t}^d \sum_{A\subseteq [d],|A|=s}(-1)^{|T\cap A|}
         \\&= \chi_T(\vec v)(-1)^{|T|}2^{-d}\sum_{s=t}^d \sum_{A\subseteq [d],|A|=s}\binom{d}{|T|}^{-1} \sum_{B\subseteq [d],|B|=|T|}(-1)^{|A\cap T|}
         \\&=\chi_T(\vec v)(-1)^{|T|}2^{-d}\sum_{s=t}^d \binom{d}{|s|}\mathcal{K}(d,s,|T|) \;,
    \end{align*}
    where in the first equality we changed the summation so that $s$ is the number of $\x_i$ that agree with $\vec v_i$, and we sum from $t$ as if $\vec v$ and $\x$ agree in more than $t$ coordinates, then the indicator is positive.  In the third inequality, we used the fact that, due to the symmetry, the sum only depends on the size of $|T|$; hence, we sum over all subsets with size $|T|$ and divide by the number of subsets with size $|T|$.
\end{proof}
\subsection{Proof of \Cref{clm:lower-bound-edge}}
We restate and prove the following: 

\clmLowerBoundEdge*
\begin{proof}
        We have that 
        \[
        \mathcal K(d,d,s)=\frac{1}{\binom{d}{s}}\sum_{A\subseteq[d],|A|=s}(-1)^{|A\cap [d]|}=(-1)^s\;,
        \]
        hence, $R_d=(2^{-d} \sum_{s=t}^d(-1)^s\binom{d}{s})^2\mathcal{K}(d,d,m)$. Therefore, $|R_d|\leq 2^{-2d}( \sum_{s=t}^d(-1)^s\binom{d}{s})^2$. Using the two identities about binomial sums, i.e., that $\sum_{s=0}^t (-1)^s\binom{d}{s}=(-1)^t\binom{d-1}{t}$ and $\sum_{s=0}^d(-1)^s\binom{d}{s}=0$, we have that $|R_d|\leq 2^{-2d}\binom{d-1}{t-1}^2$.
    \end{proof}

\subsection{Proof of \Cref{clm:lower-bound-large-k}}
We restate and prove the following claim.

\clmLowerBoundLargeK*

\begin{proof}
    Note that, $\sum_{k=0}^d\binom{d}{k}\left(2^{-d}\sum_{s=t}^d\binom{d}{s}\mathcal K(d,k,s)\right)^2=\E[f_{\vec v}^2(\x)]=\eps$. From \Cref{lem:propertiesKravchuk}, we get that $|\mathcal K(d,m,k)|\leq \exp(-c k)$, where $c>0$ is an absolute constant. Therefore, if $d/2\geq k\geq c\log(d/\eps)$ we have that $|\mathcal K(d,m,k)|\leq \eps/d$. Furthermore, using the fact that $|\mathcal K(d,m,k)|$ is symmetric with center $d/2$, we get that if $d/2\leq k\leq d-c\log(d/\eps)$, \new{then} we also have that $|\mathcal K(d,m,k)|\leq \eps/d$. Therefore, we have that 
    \begin{align*}
            \sum_{k=c\log(d/\eps)}^{d/2-c\log(d/\eps)} R_k
            &\leq \sum_{k=c\log(d/\eps)}^{d/2-c\log(d/\eps)} |R_k|
            \leq (\eps/d)\sum_{k=c\log(d/\eps)}^{d/2-c\log(d/\eps)} \binom{d}{k}\left(2^{-d}\sum_{s=t}^d\binom{d}{s}\mathcal K(d,k,s)\right)^2\leq \eps^2/d\;.
    \end{align*}
This completes the proof. 
\end{proof}

\section{Lower Bound for Low-Degree Polynomial Testing}\label{app:low-degree}

We begin by formally defining a hypothesis problem. 
\begin{definition}[Hypothesis testing] \label{def:hypothesis-testing-general}
    Let a distribution $D_0$ and a set $\mathcal{S} = \{ D_u \}_{u \in S}$ of distributions on $\R^d$. Let $\mu$ be a prior distribution on the indices $S$ of that family. We are given access (via i.i.d.\ samples or oracle) to an \emph{underlying} distribution where one of the two is true:
    \begin{itemize}
        \item $H_0$: The underlying distribution is $D_0$.
        \item $H_1$: First $u$ is drawn from $\mu$ and then the underlying distribution is set to be  $D_u$.
    \end{itemize}
    We say that a (randomized) algorithm solves the hypothesis testing problem if it succeeds with non-trivial probability (i.e.,  greater than $0.9$). 
\end{definition}
\begin{definition} \label{prob:hyp-def}
   Let $D_0$ be the joint distribution $D_0$ over the pair $(\x,y) \in \{\pm 1\}^d\times\{0,1\}$ where $\x\sim \mathcal U_d$ and $y \sim D_0(y)$ independently of $\x$. Let $D_v$ be the joint distribution over pairs $(\x,y) \in \{\pm 1\}^d\times\{0,1\}$ where the marginal on $y$ is again $D_0(y)$ but the conditional distribution $E_{\vec v}(\x|1)$ is of the form $A_{\vec v}$ (as in \Cref{thm:sq-theorem}) and the conditional distribution $E_{\vec v}(\x|0)$ is of the form $B_{\vec v}$ . Define $\mathcal S = \{E_{\vec v} \}_{{\vec v} \in S}$ for $S$ being the set of $d$-dimensional nearly orthogonal vectors from \Cref{fact:near-orth-vec-disc} and let the hypothesis testing problem  be distinguishing between $D_0$ vs. $\mathcal{S}$ with prior $\mu$ being the uniform distribution on $S$.
\end{definition}

We  need the following variant of the statistical dimension from \citet{brennan2020statistical}, which is closely related to the hypothesis testing problems considered in this section. Since this is a slightly different definition from the statistical dimension ($\mathrm{SD}$) used so far, we will assign the distinct notation ($\mathrm{SDA}$) for it.

\paragraph{Notation} For $f:\R \to \R$, $g:\R \to \R$ and a distribution $D$, we define the inner product $\langle f,g \rangle_D = \E_{X \sim D}[f(X)g(X)]$ and the norm $\snorm{D}{f} = \sqrt{\langle f,f \rangle_D }$.

\begin{definition}[Statistical Dimension] \label{def:SDA}
    For the hypothesis testing problem of \Cref{def:hypothesis-testing-general}, we define the \emph{statistical dimension} $\mathrm{SDA}(\mathcal{S},\mu, n)$ as follows:
    \begin{align*}
        \mathrm{SDA}(\mathcal{S},\mu, n) = \max \left\lbrace q \in \mathbb{N}  :  \E_{u,v \sim \mu}[| \langle  \bar{D}_u,\bar{D}_v  \rangle_{D_0} - 1 | \; | \; E] \leq \frac{1}{n} \; \text{for all events $E$ s.t. } \pr_{u,v \sim \mu}[E] \geq \frac{1}{q^2} \right\rbrace \;.
    \end{align*}
We will omit writing $\mu$ when it is clear from the context.
\end{definition}

The following lemma translates the $(\gamma,\beta)$-correlation of $\mathcal S$ to a lower bound for the statistical dimension of the hypothesis testing problem.
The proof is very similar to that of Corollary 8.28 of~\citet{brennan2020statistical} but it is given below for completeness.
\begin{lemma} \label{lem:SDA-bound}
    Let $0<c<1/2$ and $d,m \in \Z_+$. Consider the hypothesis testing problem of \Cref{prob:hyp-def}. Then, for any $q \geq 1$,
    \begin{align*}
            \mathrm{SDA}\left( \mathcal{D}, \left( \frac{\eps^{-1}\Omega(d)^{1/2-c}}{(1-2\eta)\eps(q^2 /2^{\Omega(d^{c/2})} + 1)} \right)  \right) \geq q \;.
    \end{align*}
\end{lemma}
    \begin{proof}
    The first part is to calculate the correlation of the set $\mathcal S$. By \Cref{thm:sq-theorem}, we know that the set $\mathcal{S}$ is $(\gamma,\beta)$-correlated with $\gamma = (1-2\eta)\eps^2\Omega(d)^{c-1/2} $ and $\beta = 4(1-2\eta)\eps$. 
    
    We next calculate the SDA according to \Cref{def:SDA}. We denote by $\bar{E}_{\vec v}$ the ratios of the density of $E_{\vec v}$ to the density of $R$. Note that the quantity  $\langle \bar{E}_{\vec u},\bar{E}_{\vec v} \rangle - 1$ used there is equal to $\langle \bar{E}_{\vec u} - 1,\bar{E}_{\vec v}  - 1\rangle$. Let $E$ be an event that has $\pr_{\vec u,\vec v \sim \mu}[E] \geq 1/q^2$. For $d$ sufficiently large we have that
    \begin{align*}
        \E_{u,v \sim \mu} [| \langle \bar{E}_{\vec u} ,\bar{E}_{\vec v} \rangle - 1 | E ] &\leq \min\left( 1,\frac{1}{|\mathcal{S}_{}| \pr[E]} \right) \beta 
        +\max\left( 0,1-\frac{1}{|\mathcal{S}_{}| \pr[E]} \right) \gamma \\
         &\leq (1-2\eta)\eps\left( \frac{q^2}{2^{\Omega(d^c)}}  + \frac{\eps}{\Omega(d)^{1/2-c}}\right)= (1-2\eta)\eps\left( \frac{\eps^{-1}\Omega(d)^{1/2-c}}{q^2 /2^{\Omega(d^{c/2})} + 1} \right)^{-1} 
        \;,
    \end{align*}
     where the first inequality uses that $\pr[\vec u=\vec v | E] = \pr[\vec u=\vec v , E]/\pr[E]$ and bounds the numerator in two different ways: $\pr[\vec u=\vec v , E]/\pr[E] \leq \pr[\vec u=\vec v ]/\pr[E] = 1/(|\mathcal{S}| \pr[E])$ and $\pr[\vec u=\vec v , E]/\pr[E] \leq \pr[E]/\pr[E] =1$.
\end{proof}

\subsection{Preliminaries: Low-Degree Method} \label{sec:appendix-low-degree-basics}

We begin by recording the necessary notation, definitions, and facts. This section mostly follows~\citet{brennan2020statistical}.

\paragraph{Low-Degree Polynomials} A function $f : \R^a \to \R^b$ is a polynomial of degree at most $k$ if it can be written in the   
form 
\begin{align*}
    f(x) = (f_1(x), f_2(x), \ldots, f_b(x) )\;,
\end{align*}
where each $f_i : \R^a \to \R$ is a polynomial of degree at most $k$. We allow polynomials to have random coefficients as long as they are independent of the input $x$. When considering \emph{list-decodable estimation} problems, an algorithm in this model of computation is a polynomial $f: \R^{d_1 \times n} \to \R^{d_2 \times \ell}$, where $d_1$ is the dimension of each sample, $n$ is the number of samples, $d_2$ is the dimension of the output hypotheses, and $\ell$ is the number of hypotheses returned. On the other hand, \citet{brennan2020statistical} focuses on \emph{binary hypothesis testing} problems defined in \Cref{def:hypothesis-testing-general}. 

A degree-$k$ polynomial test for \Cref{def:hypothesis-testing-general} is a degree-$k$ polynomial $f : \R^{d \times n} \to \R$ and a threshold $t \in \R$. The corresponding algorithm consists of evaluating $f$ on the input $x_1, \ldots, x_n$ and returning $H_0$ if and only if $f(x_1, \ldots, x_n) > t$.
\begin{definition}[$n$-sample $\eps$-good distinguisher]
    We say that the polynomial  $p : \R^{d\times n} \mapsto \R$ is an $n$-sample $\eps$-distinguisher 
    for the hypothesis testing problem in \Cref{def:hypothesis-testing-general} if 
    $$|{\E_{X \sim D_0^{ \otimes n }} [p(X)]  - \E_{u \sim \mu} \E_{X \sim D_{u}^{\otimes n}} [p(X)] }| \geq \eps \sqrt{ \Var_{X \sim D_0^{\otimes n}} [p(X)]} \;.$$ 
    We call $\eps$ the \emph{advantage} of the distinguisher.
\end{definition}
Let $\mathcal{C}$ be the linear space of polynomials with  a degree at most $k$. The best possible advantage is given by the \emph{low-degree likelihood ratio}\begin{equation*}
    \max_{\substack{p \in \mathcal C \\  \E_{X \sim D_0^{\otimes n}}[p^2(X)] \leq 1}}  |{\E_{u \sim \mu} \E_{X \sim D_{u}^{\otimes n}} [p(X)] - \E_{X\sim D_0^{ \otimes n }} [p(X)]}| 
    = \snorm{D_0^{ \otimes n }}{\E_{u \sim \mu}\left[(\bar{D}_u^{\otimes n})^{\leq k}\right]  -  1}\;,
\end{equation*}
where we denote $\bar{D}_u = D_u/D_0$ and the notation $f^{\leq k}$ denotes the orthogonal projection of $f$ to $\mathcal C$.

Another notation we will use regarding a finer notion of degrees is the following:  We say that the polynomial $f(x_1,\ldots, x_n) : \R^{d \times n} \to \R$ has \emph{samplewise degree} $(r,k)$ if it is a polynomial, where each monomial uses at most $k$ different samples from $x_1,\ldots, x_n$ and uses degree at most $r$ for each of them.  In analogy to what was stated for the best degree-$k$ distinguisher, the best distinguisher of samplewise degree $(r,k)$-achieves advantage $\snorm{D_0^{ \otimes n }}{\E_{u \sim \mu} [(\bar{D}_u^{\otimes n})^{\leq r,k}] - 1}$ the notation $f^{\leq r,k}$
now means the orthogonal projection of $f$ to the space of all samplewise degree-$(r,k)$ polynomials with unit norm.

\subsection{Hardness of Hypothesis Testing Against Low-Degree Polynomials} \label{appendix:hardness-of-hypothesis-testing}
We restate and prove the following.
\begin{theorem}\label{thm:hypothesis-testing-hardness}
        Let $0<c<1/2$. Consider the hypothesis testing problem of \Cref{prob:hyp-def}. For $d \in \Z_+$ with  $d$ larger than an absolute constant, any $n \leq \Omega(d)^{1/2-c}/(\eps^2(1-2\eta))$ and any even integer $k < d^{c/4}$, we have that
        \begin{align*}
            \snorm{D_0^{ \otimes n }}{\E_{\vec v \sim \mu} \left[(\bar{E}_{\vec v}^{\otimes n})^{\leq \infty,\Omega(k)}\right] - 1}^2 \leq 1\;.
        \end{align*}
    \end{theorem}
    \begin{proof}
In~\citet{brennan2020statistical}, the following relation between $\mathrm{SDA}$ and low-degree likelihood ratio is established. 
\begin{theorem}[Theorem 4.1 of~\citet{brennan2020statistical}] \label{thm:sdaldlr}
    Let $\mathcal{D}$ be a hypothesis testing problem on $\R^d$ with respect to null hypothesis $D_0$. Let $n,k \in \mathbb{N}$ with $k$ even. Suppose that for all $0\leq n' \leq n$, $\mathrm{SDA}(\mathcal{S},n') \geq 100^k(n/n')^k$. Then, for all $r$, $\snorm{D_0^{ \otimes n }}{\E_{u \sim \mu} \left[ (\bar{D}_u^{\otimes n})^{\leq r,\Omega(k)} \right] -1}^2 \leq 1$.
\end{theorem}

In \Cref{lem:SDA-bound} we set $n = {\Omega(d)^{1/2-c}/(\eps^2(1-2\eta))}$ and $q = \sqrt{2^{\Omega(d^{c/2})} (n/n')}$. Then, $\mathrm{SDA}(\mathcal{S},n') \geq  \sqrt{2^{\Omega(d^{c/2})} (n/n')} \geq (100n/n')^k$ for $k < d^{c/4}$ and then we apply the theorem above.

        \end{proof}

\end{document}